\documentclass{article}

\usepackage{algorithm}
\usepackage{appendix}
\usepackage{algorithmic}
\usepackage{graphicx, epsfig,color}
\usepackage{natbib}
\usepackage{amsmath}
\usepackage{amsthm}
\usepackage{amssymb}
\usepackage{amsfonts}
\usepackage{color}

\usepackage{todonotes}
\newtheorem{assumption}{Assumption}
\newtheorem{theorem}{Theorem}

\newtheorem{lemma}{Lemma}
\newtheorem{example}{Example}
\newtheorem{proposition}{Proposition}

\newtheorem{definition}{Definition}
\newtheorem{remark}{Remark}

\newcommand{\argmin} {{\arg\min}}

\newcommand{\x} {{\bf x}}

\DeclareMathOperator*{\sgn}{sgn}
\title{adversarial_wide_nn}

\usepackage[accepted]{aistats2022}
\begin{document}

%

%

\twocolumn[

\aistatstitle{Unlabeled Data Help: Minimax Analysis and Adversarial Robustness}

\aistatsauthor{ Yue Xing \And Qifan Song \And  Guang Cheng }

\aistatsaddress{ Purdue University\\\texttt{xing49@purdue.edu} \And  Purdue University\\\texttt{qfsong@purdue.edu} \And University of California, Los Angeles\\\texttt{guangcheng@ucla.edu} } 

]


\begin{abstract}
     
          The recent proposed self-supervised learning (SSL) approaches successfully demonstrate the great potential of supplementing learning algorithms with additional unlabeled data. However, it is still unclear whether the existing SSL algorithms can fully utilize the information of both labelled and unlabeled data. This paper gives an affirmative answer for the reconstruction-based SSL algorithm \citep{lee2020predicting} under several statistical models. While existing literature only focuses on establishing the upper bound of the convergence rate, we provide a rigorous minimax analysis, and successfully justify the rate-optimality of the reconstruction-based SSL algorithm under different data generation models. Furthermore, we incorporate the reconstruction-based SSL into the existing adversarial training algorithms and show that learning from unlabeled data helps improve the robustness.
     
\end{abstract}

\section{Introduction}

Modern learning algorithms (e.g., deep learning) have been the driving force of artificial intelligence. However, the success of these algorithms heavily relies on a huge volume of high-quality labeled training data, and these labeled data are expensive and not always available. To overcome this, the recently proposed idea of self-supervised learning (SSL) aims to supplement the training process with abundant unlabeled data, which are inexpensive and easily accessible (e.g., the images or footage captured by surveillance systems).

While SSL algorithms have gained increasing popularity, the corresponding theoretical investigations were not conducted until recent years. For instance, \cite{lee2020predicting,teng2021can} reduce and reformulate the reconstruction-based SSL as a simple statistical model and show that it improves the estimation efficiency. Two other studies \citep{arora2019theoretical,tosh2021contrastive} justify the effectiveness of contrastive learning in classification. 

However, to our best knowledge, there is no existing study working on the fundamental information limits of SSL, i.e., the minimax lower bound of the estimation efficiency. Minimax lower bound does not directly inspire new methodologies, but it helps understand whether the existing methods achieve the best or not. 

Minimax rate is the best possible convergence rate that can be achieved by \textbf{any} estimator in the worst case given finite samples, where ``the worst-case'' refers to the data distribution. 
Attaining minimax rate guarantees that the estimator achieves the best efficiency under the worst case. Failing to attain the minimax rate means that there must exist some scenarios where the estimator is not efficient. 
As a result, to understand the performance and optimality of SSL algorithms, it is essential to study the minimax bound.

Another motivation to study minimax lower bound is to understand the role of conditional independence (CI) for SSL. \cite{lee2020predicting} identifies that CI is a key factor yielding the estimation efficiency. A natural question would be how essential the CI condition is.
\cite{lee2020predicting} suggests that the convergence rate of SSL estimate might be slower when CI does not hold. However, it is unclear whether it is caused by that the SSL is inefficient under conditional dependency, or that the fundamental information limit is worse under conditional dependency. To answer this question, minimax lower bound analysis is necessary.

Besides the minimax analysis of SSL, since the aforementioned works observe a great advantage of SSL over classical supervised learning methods, it is natural to conjecture that the rationale behind SSL can potentially boost the performance of other techniques, e.g., adversarial training, differential privacy, and pruning. We consider adversarial training in this paper.

It is well known that deep learning models are vulnerable when they are fed with adversarial inputs \citep[e.g.,][]{DolphinAttack2017, papernot2017practical}. The adversarial inputs lead to serious concerns about AI safety. However, despite the existing literature in adversarial training with SSL that works on algorithm design and empirical study, e.g., \cite{kim2020adversarial,zeng2021adversarial,ho2020contrastive,gowal2020self,chen2020adversarial}, there is little theoretical understanding towards this.

In summary, this paper aims to study two problems. First, we want to establish theoretical justifications on how and when the reconstruction-based algorithm in \cite{lee2020predicting} enhances the estimation efficiency in a training process. 
Second, we study an adversarially robust adaptation of this algorithm to show its effectiveness in the adversarial setting. 

\vspace{-0.1in}
\paragraph{Model Setup} To explain the details, we denote covariates $X_1\in\mathbb{R}^{d_1}$, response $Y\in\mathbb{R}$, and some extra attributes $X_2\in\mathbb{R}^{d_2}$. Following \cite{lee2020predicting}, our target is to learn a model of $Y$ from $X_1$. We use $(X_1,X_2,Y)$ as the random variables, and $(x_1,x_2,y)$ as observations. Assume we have the following datasets:
\begin{itemize}
    \item \vspace{-0.1in}Labeled Data: $S_1$ ($n_1$ samples of $(x_1,x_2,y)$) and $S_2$ ($n_2$ samples of $(x_1,y)$).
    \item Extra unlabeled data: $S_3$ ($n_3$ samples of $(x_1,x_2)$) and $S_4$ ($n_4$ samples of $(x_1)$).
\end{itemize}
\vspace{-0.1in}
In terms of the estimation procedure, in clean training, it can be summarized as follows:
\begin{itemize}
    \item \vspace{-0.1in}Pretext task: Learn some representation  $\widehat{\phi}:=\argmin_{\phi\in\mathcal F}\frac{1}{n_1+n_3}\sum_{S_1,S_3}\|x_2-\phi(x_1)\|_2^2$ for the mapping from $x_1$ to $x_2$, where $\mathcal F$ is some function space to be defined. 
    \item Downstream task: estimate the coefficient $\widehat{W}=\argmin_{W}\frac{1}{n_1+n_2}\sum_{S_1,S_2}l(W^{\top}\widehat\phi(x_1),y)$ for some loss function $l$ depending on the specific task, e.g. regression, logistic regression. For classification task, the final classifier is $\mbox{sgn}(\widehat W^{\top}\widehat\phi(x_1))$. 
\end{itemize}\vspace{-0.1in}
To set a concrete example, we want to predict the gender $y$ of a person using the hairstyle $x_1$ and train from a set of front photos. One can first train a regression model from the hairstyle $x_1$ to predict the face $x_2$ (pretext task), then use this regression model to make a prediction for each photo, and finally, use the predicted face picture $\widehat\phi(x_1)$ to train on the gender label (downstream task). The gender label is not used in the pretext task, so we can use many unlabeled photos for the regression model in the pretext task.

In classical learning methods, the relationship between $y$ and $x_1$ is learned solely based on the training data pairs of $(x_1,y)$. The information from $x_2$ and the relationship between $(x_1,x_2)$ are overlooked. The above reconstruction-based method utilizes these data so that it potentially improves the estimation efficiency. Note that we consider two conditions for possible extra data sets besides $S_1$: (1) whether the data are labeled or not, and (2) whether the data contain $X_2$. Thus, for the sake of completeness, it is natural to include all $S_1$ to $S_4$ in our framework of minimax analysis, despite that the usual clean SSL algorithm doesn't utilize data $S_4$. We will show that $S_4$ indeed does not affect the rate of the minimax lower bound. 

Besides, as mentioned above, \cite{lee2020predicting} also discusses the importance of CI in SSL. The CI condition is formally defined as follows:
\begin{definition}
The data generation model satisfies conditional independence if $X_1$ and $X_2$ are conditionally independent given $Y=y$.
\end{definition}

\vspace{-0.1in}
\paragraph{Contributions} Our contributions are as follows:

First, we explain from a minimax perspective that the SSL method is generally well-behaved in clean training for classification models. We provide detailed characterizations in how the four datasets ($S_1$ to $S_4$) affect the estimation efficiency. From both lower bound and upper bound aspects, the size of $S_1$, $S_2$, and $S_3$ affects the efficiency, and whether SSL improves the efficiency depends on the comparison between $n_1+n_2$ and $n_1+n_3$. In addition, towards the CI condition, we reveal that when $n_1+n_3$ is large enough, no matter whether CI holds or not, SSL is minimax optimal under a well-designed family of $\mathcal{F}$, though the rate is slower without CI. As for dataset $S_4$, its sample size $n_4$ does not affect the minimax convergence rate.

Secondly, we adapt the ``pseudolabel'' method \citep{carmon2019unlabeled} so that adversarial training achieves the minimax lower bound in classification. We figure out the minimax lower bound of the convergence with the presence of unlabeled data (i.e., $S_1$ to $S_4$) and propose a way so that adversarial training achieves this lower bound with the help of SSL under the proper design of pseudolabel imputation. Again, when $n_1+n_3$ is sufficiently large, SSL improves the adversarial robustness compared to a vanilla adversarial training. 

Finally, as a by-product, we provide discussions about SSL for regression with $\phi$ being a linear function and SSL for classification with $\phi$ being a two-layer ReLU neural network (with lazy training). 
For the former one, we establish similar minimax results as the above. For the latter one, we show that the neural network family serves as a good candidate for $\mathcal{F}$ in the pretext task when there is no parametric knowledge for $\mathcal{F}$. It can potentially accelerate the convergence.

\section{Related Works}\label{sec:related_works}
Below is a summary of other related articles in the areas of self-supervised learning, adversarial training, as well as statistical minimax lower bound analysis.
\paragraph{Self-Supervised Learning} 
There are two popular types of self-supervised learning algorithms in the literature, i.e., reconstruction-based SSL and contrastive learning. Reconstruction-based SSL learns the reconstruction mapping from the large pool of unlabeled images and then employs it to reconstruct labeled images which are used in the downstream task (\citealp{noroozi2016unsupervised,zhang2016colorful,pathak2016context,doersch2015unsupervised,gidaris2018unsupervised}). Contrastive learning uses the unlabeled images to train representations that distinguish different images invariant to non-semantic transformations (\citealp{mikolov2013distributed,oord2018representation,arora2019theoretical,dai2017contrastive,chen2020simple,tian2020makes,chen2020simple,khosla2020supervised,haochen2021provable,chuang2020debiased,xiao2020should,li2020prototypical}).

\vspace{-0.1in}
\paragraph{Adversarial Training} 
Many works consider the adversarial robustness of learning algorithms from different perspectives, e.g., the statistical properties or generalization performance of the global optimum of some well-designed adversarial loss function \citep{mehrabi2021fundamental,javanmard2020precise,javanmard2020precise1,dan2020sharp,taheri2020asymptotic,yin2018rademacher,raghunathan2019adversarial,schmidt2018adversarially,najafi2019robustness,zhai2019adversarially,hendrycks2019using}, or the algorithmic properties of optimizing the adversarial loss function \citep{sinha2018certifying,gao2019convergence,zhang2020over,allen2020feature,xing2021generalization}.

Related studies about semi-supervised learning with unlabeled data can be found in deep learning and other areas. For example, \citet{carmon2019unlabeled,xing2021adversarially} verify that unlabeled data helps in improving the estimation efficiency of adversarially robust models. \citet{cannings2017local} use unlabeled data to construct the local $k$-Nearest Neighbors algorithm. 

\vspace{-0.1in}
\paragraph{Minimax Lower Bound} Minimax lower bound is an important property in the area of statistics, and has been studied for different models, e.g. non-parametric model, linear regression, LASSO, as well as adversarially robust estimate \citep{audibert2007fast,raskutti2012minimax,yang2015minimax,sun2016stabilized,dicker2016ridge,cai2016optimal,mourtada2019exact,tony2019high,dan2020sharp,xu2020cam,xing2021adversarially}.

\section{Minimax Lower Bound}\label{sec:lower_bound}

To reconcile the notation for both clean and adversarial training, for binary classification, we denote risk $R(f,\epsilon)$ as the population misclassification rate of the classifier $\mbox{sgn}[f(\widetilde X_1)]$, where $\widetilde X_1$ is the attacked input variable under strength $\epsilon$. Specifically, given $f$ and $X_1=x_1$, $\widetilde{X}_1=\arg\max_{z\in \mathcal{R}(x_1,\epsilon)}l(f(z),y)$ where $l$ is the loss function for training\footnote{In the models we consider in this paper, the attacks for $l$ and $0-1$ loss are the same.}. The constraint $\mathcal{R}(x_1,\epsilon)$ is an $\mathcal{L}_2$ or $\mathcal{L}_{\infty}$ ball centering at $x_1$ with radius $\epsilon$. Define $R^*(\epsilon)=\inf_fR(f,\epsilon)$ as the optimal misclassification rate under $\epsilon$. To train a classifier, one minimizes an empirical loss function, where the loss can be different from $R$, e.g., square loss or cross-entropy.


To regulate the distribution of $(X_1,X_2,Y)$, we impose the following assumption:
\begin{assumption}\label{assumption}
The distribution family $\mathcal{P}$ satisfies:

(1) There is some known function $p(\cdot;\cdot)$ such that, any distribution in $\mathcal{P}$ satisfies $P(Y=1|X_1=x_1)=p(x_1;\beta)$ for some $\beta\in\mathbb{R}^{d_1}$;

(2) Assume $(X_1,X_2,Y)$ satisfies (1) with $\beta=\beta^*$, then $R(2p(\cdot;\beta)-1,\epsilon)$ is $L$-Lipschitz and is twice differentiable in $\beta$ when $\beta\in B(\beta^*,r)$ for some small $r>0$ for all $\epsilon=0$ and the $\epsilon$ of interest. 

\end{assumption}

The condition (1) in Assumption \ref{assumption} is for the purpose of parametrization. Since we are doing parametric estimation, we need to consider the class of models whose parametric form exists. The generalized linear model is included by our assumption. 
The condition (2) in Assumption \ref{assumption} describes how $\beta$ is related to the misclassification rate. It should hold for $\epsilon=0$ (clean training) and the $\epsilon$ of interest in adversarial training. We fix $\epsilon$ (which does not change with $n$) and do not consider it a changing parameter throughout training.

The following theorem presents the minimax lower bounds of the convergence of any estimator when CI holds/does not hold, for both clean ($\epsilon=0$) and adversarial training ($\epsilon>0$).
Combining with upper bounds in the later section (i.e., Theorems \ref{thm:mixed:upper}, \ref{thm:linear}, \ref{thm:mix:adv}), the presented rates indeed are optimal:
\begin{theorem}\label{thm:mixed:lower:no_ci}
Assume Assumption \ref{assumption} holds. Also assume $(d_1+d_2)\log (n_i)=o(n_i^{1/3})$ for $i=1,2,3$. The minimax lower bound is
\begin{eqnarray*}
&&\inf_{\widehat f}\sup_{\mathcal{P}} \mathbb{E}R(\widehat f,\epsilon)-R^*(\epsilon)\\&=&\Omega\left(\frac{d_1}{n_1+n_2}\wedge\left(\frac{d_2}{n_1+n_2}+\frac{d_1+d_2}{n_1+n_3}\right)\right).
\end{eqnarray*}
When CI holds, the lower bound becomes
\begin{eqnarray*}
\inf_{\widehat f}\sup_{\mathcal{P}} \mathbb{E}R(\widehat f,\epsilon)-R^*(\epsilon)=\Omega\left(\frac{d_1}{n_1+n_2}\wedge\frac{d_1}{n_1+n_3}\right).
\end{eqnarray*}
\end{theorem}

The proof of Theorem \ref{thm:mixed:lower:no_ci} is postponed to the appendix. To prove the minimax lower bound, one common way is to design a specific distribution so that the distribution parameters, e.g. mean and variance, always involve error given the finite training samples. The estimator $\widehat{f}$ will further inherit this error. Our examples used to prove the minimax lower bounds in Theorem \ref{thm:mixed:lower:no_ci} are more complicated compared to \cite{dan2020sharp,xing2021adversarially}.


Although \cite{lee2020predicting} reveals that reconstruction-based SSL achieves a faster convergence rate under CI,  the minimax lower bound gets much larger when CI does not hold based on our result. From this aspect, even if CI does not hold, the SSL algorithm is still good, and achieves the optimal convergence rate based on the results in later sections.

Besides, the rates are irrelevant to the sample size of $S_4$, indicating that the information of $X_4$ is not the bottleneck for this classification problem.

\begin{remark}

\cite{xing2021adversarially} proves that introducing $S_4$ is helpful, and this does not contradict with our arguments in Theorem \ref{thm:mixed:lower:no_ci}. Assume $Var(Y|X_1,X_2)$ and $d_1$ are all constants, and $n_1=n_3=0$, then the minimax lower bound in \cite{xing2021adversarially} is still $\Omega(1/n_2)$. The unlabeled data in $S_4$ improves the convergence in a multiplicative constant level, but not the rate of the convergence. 
\end{remark}

\begin{remark}
In this paper, we consider the upper bound and lower bound of $\mathbb{E}R(\widehat{f},\epsilon)-R^*(\epsilon)$. This is different from some literature in learning theory, e.g. Section 3 in \cite{mohri2018foundations}, where they consider $R(f,\epsilon)-\sum l(f(x_1,x_2),y)/\sum n_i$. The former one focuses on the difference between the testing performance using the trained model and the true robust model, while the latter one considers the discrepancy between the training performance and the testing performance of the same model. Since we aim to study how the trained model performs compared to the true robust model, we use the former one in this paper. It is noteworthy that the latter one converges in a different rate from our results in this paper.
\end{remark}

\section{Convergence Upper Bound}\label{sec:upper_bound}
This section studies the convergence rate of SSL to see whether SSL achieves the optimal rate. 

\subsection{Convergence in Clean Training}
We translate the results in \cite{lee2020predicting} into our format to match the minimax lower bounds above.

For the pretext task, under CI, we consider learning $\phi$ from the function space $\mathcal{F}:=\{\phi \;|\; \phi(x_1)= p(x_1;\beta)\mu_2+(1-p(x_1;\beta))\mu_2',\;\beta\in\mathbb{R}^{d_1},\;\mu_2,\mu_2'\in\mathbb{R}^{d_2} \}$, where $p$ is the parametric form of $P_{\beta}(Y=1|X_1=x_1)$ as defined in Assumption \ref{assumption}.
The rationale behind this choice of $\mathcal{F}$ is that, under CI condition, $\mathbb{E}[X_2|X_1=x_1]=P(Y=1|X_1=x_1)\mathbb{E}[X_2|Y=1]+P(Y=-1|X_1=x_1)\mathbb{E}[X_2|Y=-1]$, which matches the form of functions in $\mathcal{F}$. A concrete example will be provided in Example \ref{example:mixed_gaussian} later. 

For the downstream task, we consider two estimators of $W$ as follows. For both cases, the trained classifier is defined as $\mbox{sgn}(\widehat W^\top\widehat \phi(x_1))$.
\begin{itemize}
    \item \vspace{-0.1in}Logistic regression on $(y, \widehat\phi(x_1))$.
    \item Plugin estimator in \cite{dan2020sharp}, which is equivalent to square loss in clean training. 
\end{itemize}
\vspace{-0.1in}

The following example analyzes the Gaussian mixture model when CI holds. It provides the basic analysis on how $X_2$ affects the convergence.
\begin{example}[Classification under CI]\label{example:mixed_gaussian}
Consider Gaussian mixture model defined as follows:
\begin{eqnarray*}
&&P(Y=1)=P(Y=-1)=\frac{1}{2},\\
&&\;(X_1,X_2)|Y=y\sim N\left(y\begin{bmatrix}
\mu_1^*\\\mu_2^*
\end{bmatrix},\begin{bmatrix}
\Sigma_{1,1}^* & \Sigma_{1,2}^*\\
\Sigma_{2,1}^* & \Sigma_{2,2}^*
\end{bmatrix}\right),
\end{eqnarray*}
where $\mu_i^*$'s and $\Sigma_{i,j}^*$'s are unknown parameters. The conditional distribution of $X_2$ given $X_1=x_1,Y=1$ is
\begin{eqnarray*}\label{eqn:CI}
X_2|X_1=x_1,Y=1&\sim& N( \mu_2^*+\Sigma_{2,1}^*(\Sigma_{1,1}^*)^{-1}(x_1-\mu_1^*),\\
&&\qquad\Sigma_{2,2}^*-\Sigma_{2,1}^*(\Sigma_{1,1}^*)^{-1}\Sigma_{1,2}^* ).
\end{eqnarray*}
Therefore, $\Sigma_{1,2}^*=\textbf{0}$ is equivalent to CI condition in this model. Further, the probability $P(Y=1|X_1=x_1)$ is a function of $x_1$ and $(\Sigma_{1,1}^*)^{-1}\mu_1$ only, so the best $\phi$ to minimize $\mathbb{E}\|x_2-\phi(x_1)\|^2$ can be represented as
$\phi^*(x_1)
= (2p(x_1;(\Sigma_{1,1}^*)^{-1}\mu_1^*)-1)\mu_2^*$ under CI. 
Based on this, the family of $\phi$, $\mathcal{F}=\{\phi\mid \phi(x)=(2p(x;\Sigma_{1,1}^{-1}\mu_1)-1)\mu_2,\;\forall \Sigma_{1,1}^{-1}\mu_1,\mu_2 \}$ is a proper choice for the pretext task.

Solving the pretext task, we have
\begin{eqnarray*}
&&\widehat{\phi}(x_1) = (2p(x_1;\widehat{\Sigma_{1,1}^{-1}\mu_1})-1)\widehat\mu_2,\text{ where}\\
&&\mathbb{E}\|\widehat{\Sigma_{1,1}^{-1}\mu_1}-(\Sigma_{1,1}^*)^{-1}\mu_1^* \|^2=O\left(\frac{d_1}{n_1+n_3}\right).
\end{eqnarray*}
The detailed derivation is postponed to appendix. 

In the downstream task, since the output of $\widehat{\phi}$ is always in the same direction (parallel to $\widehat\mu_2$), $\widehat{W}\widehat\phi(x_1)$ becomes $c(2p(x_1;\widehat{\Sigma_{1,1}^{-1}\mu_1})-1)$ for some constant $c$, and its sign only depends on $(2p(x_1;\widehat{\Sigma_{1,1}^{-1}\mu_1})-1)$. The estimation error in $\widehat{W}$ and $\widehat{\mu}_2$ therefore does not affect the final prediction, and the error in the prediction is only caused by the error in $\widehat{\Sigma_{1,1}^{-1}\mu_1}$. Consequently, 
\begin{eqnarray}\label{eqn:example:mixed}
\mathbb{E}R(\widehat{W}^\top\widehat\phi,0)-R^*(0) = O\left(\frac{d_1}{n_1+n_3}\right).
\end{eqnarray}

\end{example}
The proof for Example \ref{example:mixed_gaussian}, and Theorem \ref{thm:mixed:upper} and Theorem \ref{thm:linear} below are postponed to the appendix. The basic idea is to use Taylor expansion on the estimation equation to obtain the Bahadur representation of the estimator. In general, via Bahadur representation, one can show that the estimator asymptotically converges to the true model in Gaussian.

The following theorem can be obtained via extending Example \ref{example:mixed_gaussian} to other models under CI:

\begin{theorem}\label{thm:mixed:upper}
Assume Assumption \ref{assumption} together with some finite-variance condition (to be specified in the appendix) hold. If $(d_1+d_2)\log (n_i)=o(n_i^{1/3})$ for $i=1,2,3$, and $d_2=o(\sqrt{d_1(n_1+n_3)})$, then for both the two loss functions (logistic loss and square loss), when CI holds, 
\begin{eqnarray*}
\mathbb{E}R(\widehat{W}^\top\widehat\phi,0)-R((W^*)^{\top}\phi^*,0) = O\left(\frac{d_1}{n_1+n_3}\right),
\end{eqnarray*}
where $\phi^*$ is the population loss minimizer of the pretext task, and $W^*$ is the population loss minimizer (logistic or square correspondingly) in the downstream task.
\end{theorem}

In contrast to Theorem \ref{thm:mixed:upper}, the following theorem studies the convergence of SSL when CI does not hold. For simplicity, we consider using linear $\phi$ in the pretext task, i.e., $\mathcal F$=\{linear mappings from $\mathbb R^{d_1}$ to $\mathbb R^{d_2}$\}
\begin{theorem}\label{thm:linear}
Assume Assumption \ref{assumption} together with some finite-variance condition (to be specified in the appendix) hold. If $(d_1+d_2)\log (n_i)=o(n_i^{1/3})$ for $i=1,2,3$. For linear $\phi$, if the singular values of $\mathbb{E}X_1X_2^{\top}$ are finite and bounded away from zero, then
\begin{eqnarray*}
&&\mathbb{E} R(\widehat{W}^\top\widehat\phi,0) - R((W^*)^{\top}\phi^*,0) \\
&=& O\left( \frac{ d_2}{n_1+n_2}+\frac{d_1+d_2}{n_1+n_3} \right),
\end{eqnarray*}
where $\phi^*$ is the minimizer of the population loss of the pretext task, and $W^*$ is the minimizer of the population loss in the downstream task.
\end{theorem}



Together with the lower bounds obtained in Theorem \ref{thm:mixed:lower:no_ci}, the upper bounds in Theorem \ref{thm:mixed:upper} and \ref{thm:linear} indicate that SSL achieves minimax optimal for clean training, when $n_1+n_3\gg n_1+n_2$ under CI, or $n_1+n_3\gg (d_1+d_2)(n_1+n_2)/(d_1-d_2)$ without CI. This implies that SSL efficiently utilizes data information to achieve optimal convergence, while the deterioration of rate when CI fails is merely caused by information bottleneck of the data.

Furthermore, for commonly used model which only considers $(X_1,Y)$, the upper bound is $O(d_1/(n_1+n_2))$, e.g. \cite{xing2021generalization}. Compared to this rate, the upper bounds in the Theorem \ref{thm:mixed:upper} and \ref{thm:linear} are faster when $n_3$ is large. These observations imply that the reconstruction-based SSL does perform better than only studying the relationship between $X_1$ and $Y$. 

\vspace{-0.1in}
\paragraph{Simulation Study} We use the model in Example \ref{example:mixed_gaussian} to numerically verify the effectiveness of SSL under CI. We take $d_1=5$, $d_2=2$, the mean vector $\mu=\textbf{1}_{d_1+d_2}/\sqrt{d_1}$, and the covariance matrix $\Sigma=I_{d_1+d_2}$. We repeat 100 times to obtain the mean and variance of Regret. The sample size $n_2$ and $n_4$ are zero, and $n_1=100$. The results for plugin estimate are summarized in Table \ref{tab:clean:ssl}. From Table \ref{tab:clean:ssl}, SSL improves the performance when $n_3$ is large enough. The observations in logistic regression are similar (postponed to the appendix).

\begin{table*}[]
\centering
\caption{Regret in clean training under CI: SSL (plugin estimator/square loss) vs learning only labeled data.}\label{tab:clean:ssl}
\begin{tabular}{ccccc}
\hline
 $n_3$    & SSL (mean) & labeled (mean) & SSL (var)   & labeled (var)  \\\hline
500   & 0.01057 &  0.00959 & 7.07E-05 &  3.84E-05 \\
1000  & 0.00529 &  0.01017 & 1.83E-05 &  7.53E-05 \\
5000  & 0.00104 &  0.00970 & 2.01E-06 &  4.78E-05 \\
10000 & 0.00042 &  0.00876 & 1.55E-06 &  4.88E-05 \\
20000 & 0.00031 &  0.00974 & 6.94E-07 &  5.20E-05\\\hline
\end{tabular}

\end{table*}

\subsection{Adversarial Training}

Intuitively, a straightforward way to adapt SSL in adversarially robust learning is to perform the downstream task with adversarial loss. However, a simple example below illustrates that such a procedure may lead to a bias:

\begin{example}
Under the Gaussian mixture classification model in Example \ref{example:mixed_gaussian}, following \cite{dan2020sharp}, one can show that the population adversarial risk minimizer, for both the two loss functions, is a linear classifier whose coefficient vector is of the form $(A\Sigma_{1,1}^* +BI_{d_1})^{-1}\mu_1^*$. Using $\mathcal{F}$ considered in Example \ref{example:mixed_gaussian}, the decision boundary implied from $W^\top\widehat{\phi}$ (for any $W$) is always parallel to $\widehat{(\Sigma_{1,1})^{-1}\mu_1}$ which is a biased estimation for $(A\Sigma_{1,1}^* +BI_{d_1})^{-1}\mu_1^*$ if $\Sigma_{1,1}^*$ is not proportional to $ I_{d_1}$ .
\end{example}

To ensure the consistency of the adversarially robust estimator, one can borrow the idea of \cite{carmon2019unlabeled,uesato2019labels}: we first use SSL in clean training, and based on which, we create pseudolabel for data in  $S_3$ and $S_4$, then we perform an adversarial training using $S_1$ to $S_4$ with the pseudolabels. Algorithm \ref{alg:simple} summarizes this procedure.
\begin{algorithm}
	\caption{Adversarial Training with SSL (adv+SSL)}\label{alg:simple}
		\begin{algorithmic}
			\STATE {\bfseries Input:} data $S_i$ for $i=1,...,4$. Adversarial training configuration
$(\eta,T,\epsilon,...)$.
			\STATE Use $S_1$ and $S_3$ to obtain $\widehat{\phi}$.
			\STATE Use $S_1$ and $S_2$ to obtain $\widehat{W}$.
			\STATE Create pseudolabel for samples in $S_3$ and $S_4$ as $\widehat{y}$. Take $\widehat{y}=y$ for $S_1, S_2$.
			\STATE Conduct adversarial training with configuration $(\eta,T,\epsilon,...)$ to obtain $\widehat{\theta}$ where
			$$\widehat{\theta}=\argmin_{\theta}\frac{1}{\sum n_i}\sum_{S_1,...,S_4}\max_{\widetilde{x}_1\in \mathcal{R}(x_1,\epsilon)}l(\theta^{\top}\widetilde{x}_1,\widehat{y}).$$
			\STATE{\bfseries Output:}   the robust model $\widehat{\theta}$.
		\end{algorithmic}
 	\end{algorithm}
 	
With an abuse of notation, we denote $R(\theta,\epsilon)$ as the risk of the linear classifier $\mbox{sgn}(\theta^\top \widetilde X_1)$.
Algorithm \ref{alg:simple} focuses on the linear classifier $\theta^\top\widetilde{x}_1$, but in real practice, one may train a nonlinear model in the adversarial training stage (e.g., a neural network).

\vspace{-0.1in}
\paragraph{How to obtain reasonable pseudolabels?} A key requirement of the pseudolabel is that the distribution of $(X_1,\widehat{Y})$ approximately matches $(X_1,Y)$. Thus $\widehat Y$ is a simple plug-in estimator $p(x_1;\widehat\beta)$ when we have the parametric form of $p$, i.e., $P(\widehat Y=1|X_1=x_1)$.  
The following Gaussian mixture model example illustrates how to construct pseudolables for unlabeled data:

\begin{example}[Pseudolabel for Gaussian Mixture Model]\label{example:mixed:label} 
 When estimating $\phi$, we are considering the class of function $\mathcal{F}=\{\phi\mid \phi(x_1)=(2p(x_1;\Sigma_{1,1}^{-1}\mu_1)-1)\mu_2,\;\forall \Sigma_{1,1}^{-1}\mu_1,\mu_2 \}$. An estimate of $P(Y=1|X_1=x_1)$, i.e. $p(x_1;\widehat{\Sigma_{1,1}^{-1}\mu_1})$, can be directly obtained from the pretext task. This construction method can also be applied to general models in $\mathcal{P}$.
\end{example}
The following theorem evaluates the convergence rate of Algorithm \ref{alg:simple} and shows its effectiveness:

\begin{theorem}\label{thm:mix:adv}
Assume Assumption \ref{assumption} and some finite-variance condition (in the appendix) hold.

(I) Assume the conditions in Theorem \ref{thm:mixed:upper} hold.
Denote $\theta^*=\argmin_{\theta} El(\theta^{\top}\widetilde{X}_1,{Y})$ as the optimal linear classifier using square loss/logistic regression.  Denote $\widehat\theta$ as the linear adversarially robust estimator obtained via Algorithm \ref{alg:simple}. If $\widehat{W}^{\top}\widehat{\phi}$ is unbiased, then for square loss/logistic regression, 
\begin{eqnarray*}
\mathbb{E}R(\widehat\theta,\epsilon)-R(\theta^*,\epsilon)= O\left(\frac{d_1}{n_1+n_3}\right).
\end{eqnarray*}
(II) Assume the conditions in Theorem \ref{thm:linear} hold, if $\widehat{W}^{\top}\widehat{\phi}$ is asymptotically unbiased, 
\begin{eqnarray*}
\mathbb{E}R(\widehat\theta,\epsilon)-R(\theta^*,\epsilon)= O\left(\frac{d_2}{n_1+n_2}+\frac{d_1+d_2}{n_1+n_3}\right).
\end{eqnarray*}
\end{theorem}
The proof of Theorem \ref{thm:mix:adv} is postponed to the appendix. To build the connection between the clean training and the adversarial training, we borrow the idea from  semi-parametric problems to expand $p(x_1;\widehat{\Sigma_{1,1}^{-1}\mu_1})$ in the Taylor expansion for the estimation equation, e.g., \cite{wang2009locally}.

Theorem \ref{thm:mix:adv} shows the convergence rate of the estimator obtained in Algorithm \ref{alg:simple}. Again for adversarial training, SSL achieves minimax optimal when $n_3$ is large.

\vspace{-0.1in}
\paragraph{Effect of Accuracy of Imputed Labels} Theorem \ref{thm:mix:adv} establishes the convergence rate of the whole procedure in Algorithm \ref{alg:simple} where the SSL clean training stage helps estimate probability $p=P(Y=1|X_1=x_1)$. However, in real practice, when there is no parametric knowledge of the model (i.e., Assumption \ref{assumption} fails), it is not easy to obtain an accurate $\widehat{p}$, and people may consider directly using the predicted label as the pseudolabel. The following result illustrates how the accuracy of $\widehat{p}$ affects the convergence in logistic regression. For square loss, the condition is slightly different, and we postpone the discussion to the appendix.
\begin{proposition}\label{prop:mix:adv}
Under the conditions of Theorem \ref{thm:mix:adv}(I), assume one obtains some consistent $\widehat{p}$ such that $\mathbb{E}\|X_1\|^2\|\widehat{p}(X_1)-p(X_1)\|^2  \rightarrow 0$ in $n_1+n_3$, then for logistic regression, (1) $\widehat{\theta}$ is consistent to $\theta^*$; and (2) the convergence rate of $\widehat{\theta}$ is $O( d_1/(\sum n_i) + \mathbb{E}\|X_1\|^2\|\widehat{p}(X_1)-p(X_1)\|^2 )$.
\end{proposition}
\vspace{-0.1in}
\paragraph{Simulation Study} Our aim is to numerically verify: (1) Algorithm \ref{alg:simple} improves the overall performance; and (2) the dataset $S_4$ is not the bottleneck of the convergence, which is an observation from the comparison among upper bounds and lower bounds as discussed Section \ref{sec:lower_bound}. Similar to Table \ref{tab:clean:ssl}, we take  $d_1=5$, $d_2=2$. The mean and variance are $\mu=\textbf{1}_{d_1+d_2}/\sqrt{d_1}$, $\Sigma=I_{d_1+d_2}$ respectively. We consider $\mathcal{L}_2$ attack with $\epsilon=0.1$ in this experiment. 

\begin{table*}[]
\centering
\caption{Average Regret of adversarially robust estimate under CI condition. $n_1=100$. The variance information is in Table \ref{tab:adv:ssl:var} in the appendix. }\label{tab:adv:ssl}
\begin{tabular}{cccccc}
\hline  $n_3$    & benchmark       & adv+SSL($S_1$,$S_3$)              &  adv($S_1$)             & adv+SSL($S_1$,$S_3$,$S_4$)         & adv+pseudo label($S_1$,$S_3$)             \\\hline
  500   & 0.00771 & 0.01264 & 0.01041 & 0.01211 & 0.00965 \\
  1000  & 0.00543 & 0.01195 & 0.01040 & 0.01046 & 0.00946 \\
  5000  & 0.00150 & 0.00543 & 0.00897 & 0.00526 & 0.00916 \\
  10000 & 0.00070 & 0.00332 & 0.01008 & 0.00330 & 0.00898 \\
  20000 & 0.00050 & 0.00213 & 0.00956 & 0.00185 & 0.00907    \\\hline             
\end{tabular}

\end{table*}

In Table \ref{tab:adv:ssl}, the benchmark algorithm is omnipotent and performs standard adversarial training on $S_1$ to $S_3$, with labels in $S_3$ known. For benchmark and the other methods except for adv+SSL($S_1,S_3,S_4$), they do not use $S_4$, while for the method adv+SSL($S_1,S_3,S_4$), $n_4=n_3$. The method adv($S_1$) means to use adversarial training on the dataset $S_1$ only. The method adv+pseudo label($S_1$,$S_3$) is to use clean training in $S_1$ to impute the label for samples in $S_3$ and then conduct adversarial training.

From Table \ref{tab:adv:ssl}, we have some observations.

First, the quality of the imputed labels affects the adversarial robustness, and unlabeled data helps improve adversarial robustness. Comparing the benchmark and the other methods, the benchmark has a better clean training stage result, i.e., the true label, thus the final adversarial estimate is better. Comparing adv+SSL($S_1$,$S_3$),  adv($S_1$), and adv+pseudo label($S_1$,$S_3$), we claim that using unlabeled data helps improve the estimation efficiency. 

In addition, $S_4$ only slightly contributes to the improvement of adversarial robustness. Comparing adv+SSL($S_1$,$S_3$) and adv+SSL($S_1$,$S_3$,$S_4$), we see that the additional data $S_4$ do not significantly improve the estimation efficiency. Similar observations can be found in Table \ref{tab:adv:ssl:log} for logistic regression (in appendix).

\section{Additional Discussions}\label{sec:discussion}
We provide some additional discussions as by-products of the analysis above. In the main text, we provide theoretical results associated with two-layer neural networks in SSL. Due to the space limit, we postpone the discussion about the linear regression model with linear $\phi$  to the appendix. 

\subsection{Neural Networks}
\label{sec:mixed:nn}
The design of $\phi$ in the previous sections is based on the parametric knowledge of the data generating model. We consider using a two-layer neural network as a ``nonparametric'' alternative to model $\phi$ while such knowledge is unavailable. In the literature, there are abundant results on the expressibility or fitting convergence of neural networks, e.g. \cite{schmidt2020nonparametric,bauer2019deep,elbrachter2019deep,hu2020optimal,hu2021regularization,farrell2021deep}. 

We follow \cite{hu2021regularization} to consider an easy-to-implement estimation procedure. To be specific, we use a two-layer neural network $\phi=(\phi_1,\dots,\phi_{d_2})$ with
\begin{eqnarray*}
&& \phi_k(x)=\frac{1}{\sqrt{m}}\sum_{j=1}^m a_{j,k} \sigma(w_{j,k}^{\top}x), \quad w_{j,k}\in\mathbb{R}^{d_1},
\end{eqnarray*}
where $a_{j,k}$ are generated from $\{\pm 1\}$ uniformly. The weights $w_{j,k}$'s are initialized from $N(0,\tau^2 I_{d_1})$ and trained with an $\mathcal{L}_2$ penalty with multiplier $\lambda$.

\begin{proposition}\label{prop:nn}
Let $d_1$ and $d_2$ be fixed and CI holds. Assume $\widehat{\phi}$ is trained under proper configurations and the distribution of $(X_1,Y)$ satisfies some extra conditions. Assume $n_3$ is large enough so that $n_3\gg \text{poly}(n_1+n_2)$ for some polynomial of $n_1+n_2$. Both $\widehat{W}$ and $\widehat{\phi}$ in the SSL procedure are regression estimator. Then with high probability,
\begin{eqnarray*}
R(\widehat{W}\widehat{\phi},0)-R^*(0)=o\left(\frac{1}{n_1+n_2}\right).
\end{eqnarray*}
\end{proposition}
The detailed conditions in Proposition \ref{prop:nn} are postponed to Appendix \ref{sec:appendix:proof3}. From Proposition \ref{prop:nn}, when there are sufficient samples in $S_3$, the SSL procedure will improve the accuracy in clean training even if we do not have parametric knowledge for the model. 
\begin{figure*}
    \centering
    \includegraphics[scale=0.45]{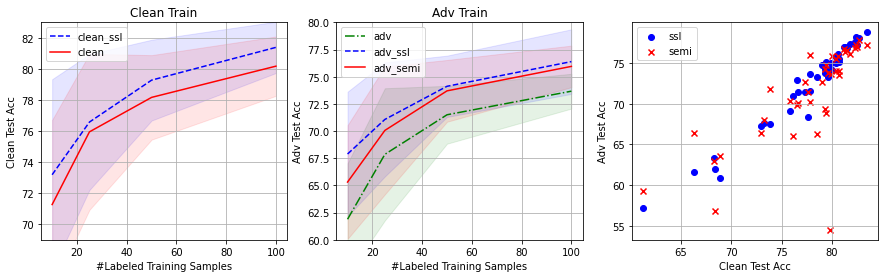}
\caption{Clean and adversarial ($\mathcal{L}_{\infty}$ attack with $\epsilon=4/255$) test accuracy (gender classification) using Yearbook data. Left: clean training accuracy. clean\_ssl: SSL clean training. clean: benchmark clean training. Middle: adversarial training accuracy adv: standard adversarial training with only labeled data. adv\_ssl: adv+SSL algorithm. adv\_semi: \cite{carmon2019unlabeled}. Right: clean test accuracy in clean training vs adversarial test accuracy in adversarial training.   
}\label{tab:yearbook}
\end{figure*}

\vspace{-0.1in}
\paragraph{Real-Data Experiment}
We use the Yearbook dataset from \cite{ginosar2015century}. We consider a two-layer ReLU network with lazy training (this network matches Proposition \ref{prop:nn}) for $\phi$. Since all of our theoretical results are developed under large-sample asymptotics, we resize the images to 32x32, take the center 16x16 patch of an image as $X_2$, and take the rest as $X_1$. The goal is to classify the gender of each image. We minimize the square loss to obtain a classifier.

For the clean training task, in the pretext task, we randomly select 20,000 samples and regress $X_2$ on $X_1$ to learn the representation mapping $\widehat\phi$ using a two-layer ReLU network with lazy training and taking $m$, the number of hidden nodes, as $1000$. Since the data dimension $d_1$ and $d_2$ are comparable to 20,000, we add an $\mathcal{L}_2$ penalty in the regression loss. The pretext task is trained by 100 epochs. For the downstream task, we take $10,25,50,100$ samples with their labels to obtain $\widehat{W}$. Using square loss, there is an analytical solution of $\widehat{W}$. We add an $\mathcal{L}_2$ penalty to the square loss and tune it to achieve the best prediction accuracy. For adversarial training task, 
we use SSL in clean training and impute labels for unlabeled data,
 and use all data with labels/pseudolabels to train a two-layer ReLU network (with 1000 hidden node) as the adversarial classifier.  We use $\mathcal{L}_{\infty}$ attack with $\epsilon=4/255$.

To assess the performance of SSL for the clean training task, we compare it against a benchmark clean training algorithm: we directly use the two-layer ReLU network with 1000 hidden nodes (train all layers) on the labeled samples for 100 epochs with a learning rate of 0.1. For fairness, we also add and tune the $\mathcal{L}_2$ penalty to achieve the best performance. To assess the performance of adv+SSL algorithm for the adversarial training task, we compare it against two benchmark adversarial training algorithms: (1) We conduct the standard clean training using only the labeled data $(x_1,y)$ and then impute labels for unlabeled data to do adversarial training, i.e., the exact algorithm in \cite{carmon2019unlabeled}. (2) We only use the given labeled data to conduct adversarial training and do not use the unlabeled data. We use a two-layer ReLU network (with 1000 hidden nodes) for both benchmarks.

The experiment results are summarized in Figure \ref{tab:yearbook}, based on the average and variance of testing accuracies of over ten repeated runs. There are three figures in Figure \ref{tab:yearbook}. The left panel of Figure \ref{tab:yearbook} compares the clean testing accuracy, and it is easy to see that SSL leads to a higher accuracy than the benchmark method. The middle panel of Figure \ref{tab:yearbook} compares the adversarial testing accuracy, and one can see that utilizing unlabeled data helps improve the testing performance. 
If we compare the blue dashed curve and the red curve in both left and middle plots, it suggests that a better clean model (used to generate pseudolabels) leads to a better adversarial model. To confirm this, in the right panel of Figure \ref{tab:yearbook}, we plot the clean testing accuracy in clean training against the adversarial testing accuracy in adversarial training to study how the quality of the pseudolabels affects the final adversarial robustness. One can observe a positive correlation between these two accuracies, implying a positive correlation between the pseudolabel quality and final adversarial robustness. We conjecture that the improvements in Figure \ref{tab:yearbook} is not as remarkable as Table \ref{tab:adv:ssl} in simulation due to (1) from Proposition \ref{prop:nn}, the required $n_3$ is much larger than Theorem \ref{thm:mix:adv}, and (2) the data dimension $d_2$ for real data is larger than simulated model, involving more error in estimating $\widehat\phi$.

\section{Conclusion}

In this paper, we investigate the statistical properties of reconstruction-based SSL. In particular, we study the minimax lower bound of estimation accuracy and the adversarial robustness. Through figuring out these properties, we argue that (1) in clean training, no matter CI holds or not, reconstruction-based SSL reaches the optimal rate of convergence in the models we consider; and (2) it is possible to design adversarially robust estimate such that it is also optimal. These advantages of the SSL method lead to a better performance of SSL compared with the vanilla training.

There are several potential directions for future development. First, the procedure for adversarial training considered is tedious (i.e., pretext task and downstream task in clean training, and the adversarial training itself). It is interesting to simplify the procedure.
Second, we only provide some light discussion using neural networks for the pretext task when the parametric form of $P(Y=1|X_1=x_1)$ is unknown. An in-depth investigation on this matter is definitely worthwhile, as deep neural networks have been viewed as a powerful nonparametric learning tool for modern data sciences.
Third, this paper only discusses reconstruction-based SSL, and it is of great interest to generalize our analysis to other types of SSL-based methods.

\section{Acknowledgements}
This project is partially supported by NSF-SCALE MoDL (2134209).

\bibliographystyle{asa}
\bibliography{regression}

\newpage
\onecolumn
\appendix

Below is the list of contents in the appendix:
\begin{itemize}
    \item Section A: discussion about regression.
    \item Section B: additional experiments, extra tables, and results regarding to neural networks.
    \item Section C: additional assumptions on the finite variance, and lemmas which provide minimax lower bound for some particular distributions.
    \item Section D: proofs for results in Section 2 and 3.
    \item Section E: proofs for results in Section 4 and A.
    \item The derivations for different losses are similar, so during the proofs, we firstly present the proof details for one loss, then display the proof for the other losses in a separate proof block to mention the differences. 
\end{itemize}

\section{Regression}\label{sec:appendix:regression}
For regression, we assume $X_1$ and $X_2$ jointly follow some multivariate Gaussian distribution with $Var(X_1)=I_{d_1}$, $Var(X_2)=I_{d_1}$. The singular values of $Cov(X_1,X_2)$ are finite and bounded away from zero. The response satisfies $y=\theta_0^{\top}x_1+\varepsilon$ for noise $\varepsilon$. The following theorem presents the convergence rate of the SSL estimate and the minimax lower bound:
\begin{theorem}\label{thm:regression}
For linear regression model described in above, assume $\theta_0=Cov(X_1,X_2)a_0$, then
\begin{eqnarray*}
\mathbb{E} R(\widehat{W}^{\top}\widehat{\phi},\epsilon) - R^*(\epsilon) = O\left( \frac{\sigma^2 d_2}{n_1+n_2}+\frac{\|a_0\|^2(d_1+d_2)}{n_1+n_3} \right),
\end{eqnarray*}
which is minimax optimal when $n_1+n_3\gg (n_1+n_2)\|a_0\|^2(d_1+d_2)/(\sigma^2(d_1-d_2))$.
\end{theorem}
Note that the condition $\theta_0=Cov(X_1,X_2)a_0$ implies that the final estimate $\widehat{W}^{\top}\widehat{\phi}$ is asymptotically unbiased to $\theta_0$. In comparison, Theorem \ref{thm:regression} delivers a similar conclusion to the Theorem \ref{thm:mix:adv}(II). Although the convergence rate of SSL estimate is slower without CI, it still reaches the minimax lower bound. The proof of Theorem \ref{thm:regression} is postponed to Section \ref{sec:appendix:proof3}. 

We only present the case where CI fails in Theorem \ref{thm:regression} as CI is not appropriate in the model we consider. The following example illustrates this issue:
\begin{example}\label{example:zero_variance}
Assume $(X_1,X_2,\varepsilon)$ follows multivariate Gaussian, then CI implies $\sigma^2=0$.
\end{example}

\section{Extra Tables and Additional Experiments}\label{sec:appendix:table}

\begin{table*}[!ht]
\centering
\caption{Variance table corresponding to Table \ref{tab:adv:ssl}}\label{tab:adv:ssl:var}
\begin{tabular}{cccccc}
\hline  $n_3$    & benchmark      & adv+SSL($S_1$,$S_3$)              &  adv($S_1$)             & adv+SSL($S_1$,$S_3$,$S_4$)         & adv+pseudo label($S_1$,$S_3$)              \\\hline
  500   & 1.14E-05 & 3.55E-04 & 7.09E-05 & 3.18E-04 & 3.53E-05 \\
  1000  & 8.43E-06 & 8.35E-05 & 6.71E-05 & 9.69E-05 & 3.07E-05 \\
  5000  & 7.85E-07 & 8.68E-06 & 4.58E-05 & 8.62E-06 & 2.27E-05 \\
  10000 & 3.01E-07 & 4.61E-06 & 6.12E-05 & 5.65E-06 & 1.62E-05 \\
  20000 & 2.06E-07 & 2.48E-06 & 5.11E-05 & 1.42E-06 & 1.79E-05  \\\hline  
\end{tabular}
\end{table*}

\begin{table*}[!ht]
\centering
\begin{tabular}{ccccc}
\hline n3    & SSL(mean) & labeled(mean) & SSL(var)   & labeled(var)  \\\hline
500   & 0.00946 & 0.02307 & 4.17E-05 & 3.19E-04 \\
1000  & 0.00458 & 0.02076 & 1.24E-05 & 3.14E-04 \\
5000  & 0.00106 & 0.02304 & 1.84E-06 & 3.42E-04 \\
10000 & 0.00049 & 0.02375 & 1.25E-06 & 2.72E-04 \\
20000 & 0.00015 & 0.02394 & 6.79E-07 & 2.23E-04 \\\hline
\end{tabular}
\caption{Clean estimate under CI condition using logistic regression. $n_1=100$. Left: mean regret, right: the variance.}\label{tab:clean:log}
\end{table*}

Table \ref{tab:clean:log} is the clean training result using logistic regression, and the observations are similar to Table \ref{tab:clean:ssl}. Table \ref{tab:adv:ssl:log} is the adversarial training result using logistic regression, and the observations are similar to Table \ref{tab:adv:ssl}. Table \ref{tab:adv:ssl:log:var} summarizes the variance information for simulation in adversarial training.

\begin{table*}[!ht]
\centering
\begin{tabular}{cccccc}
\hline $n_3$     & benchmark       & adv+SSL($S_1$,$S_3$)              &  adv($S_1$)             & adv+SSL($S_1$,$S_3$,$S_4$)         & adv+pseudo label($S_1$,$S_3$)             \\\hline
  500   & 0.00218 & 0.00570 & 0.01229 & 0.00441 & 0.00643 \\
  1000  & 0.00104 & 0.00325 & 0.01209 & 0.00291 & 0.00572 \\
  5000  & 0.00021 & 0.00064 & 0.01254 & 0.00072 & 0.00477 \\
  10000 & 0.00025 & 0.00051 & 0.01248 & 0.00027 & 0.00485 \\
  20000 & 0.00002 & 0.00020 & 0.01245 & 0.00017 & 0.00553    \\\hline
\end{tabular}
\caption{Mean: Difference between the adversarial testing accuracy of estimators and $\theta^*$ using logistic regression. $n_1=100$.  The variance information is in Table \ref{tab:adv:ssl:log:var}.}\label{tab:adv:ssl:log}
\end{table*}

\begin{table*}[!ht]
\centering
\begin{tabular}{cccccc}
\hline    $n_3$     & benchmark      & adv+SSL($S_1$,$S_3$)              &  adv($S_1$)             & adv+SSL($S_1$,$S_3$,$S_4$)         & adv+pseudo label($S_1$,$S_3$)              \\\hline
 500   & 3.78E-06 & 1.63E-05 & 7.20E-05 & 1.29E-05 & 2.09E-05 \\
 1000  & 1.08E-06 & 7.06E-06 & 9.72E-05 & 5.30E-06 & 2.06E-05 \\
 5000  & 4.06E-07 & 7.65E-07 & 8.69E-05 & 5.96E-07 & 9.63E-06 \\
 10000 & 1.99E-07 & 4.28E-07 & 9.69E-05 & 2.86E-07 & 9.09E-06 \\
 20000 & 1.32E-07 & 2.72E-07 & 8.62E-05 & 2.20E-07 & 2.07E-05  \\\hline
\end{tabular}
\caption{Variance table corresponding to Table \ref{tab:adv:ssl:log}}\label{tab:adv:ssl:log:var}
\end{table*}

\section{Lemmas and Extra Conditions}

We first introduce some extra conditions and lemmas which are related to the minimax lower bounds. The extra conditions are technical assumptions regulating the behavior of the loss (and its Taylor expansions) to ensure a finite variance. The lemmas are some particular examples used in the minimax lower bound.

\begin{assumption}\label{assumption:more}
We further assume that $\mathcal{P}$ satisfies:
\begin{enumerate}
    \item[(3)] $\mathbb{E}\|X_1\|^2X_1X_1^{\top}/\mathbb{E}\|X_1\|^2$ have bounded eigenvalues.
    \item[(4)] When CI holds, all the eigenvalues of the covariance of $\partial \|X_2-\mathbb{E}(X_2|X_1)\|^2/\partial\beta$ and the expectation of $\partial^2 \|X_2-\mathbb{E}(X_2|X_1)\|^2/\partial\beta^2$ are of $\Theta(\gamma_{d_2})$ for some $\gamma_{d_2}=\Omega(1)$.
    \item[(5)] When CI holds, if $(X_1,X_2,Y)$ follows distribution $P_{\beta^*}$ for some $\beta^*$, then $\mathbb{E}\partial \|X_2-\mathbb{E}(X_2|X_1)\|^2/\partial\beta$ is $L$-Lipschitz for some constant $L>0$ when $\beta\in B(\beta^*,r)$ where $B$ denotes an $L_2$ ball and $r$ is some constant.
\end{enumerate}

\end{assumption}

\begin{assumption}\label{assumption:extra}
We assume that
\begin{enumerate}
    \item when $l$ is logistic regression: the covariance of $\partial l/\partial\theta$ has all bounded and greater-than-zero eigenvalues, the expectation of $\partial^2 l/\partial\theta^2$ has all bounded and greater-than-zero eigenvalues. The density is finite and away from zero when $x_1$ is near the decision boundary. In addition, all the eigenvalues of 
    \begin{eqnarray*}
    \mathbb{E}_{X_1}\|X_1\|^2\left[ \left(\frac{\partial p(X_1;\beta^*)}{\partial\beta}\right)^{\top}\mathbb{E}_{X_1}\left\{\left(\frac{\partial p(X_1;\beta^*)}{\partial\beta}\right)\left(\frac{\partial p(X_1;\beta^*)}{\partial\beta}\right)^{\top}\right\}^{-1}\left(\frac{\partial p(X_1;\beta^*)}{\partial\beta}\right) \right]^2
    \end{eqnarray*}
    are in $\Theta(d_1^3)$. 
    \item when $l$ is square loss: the covariance of $\partial l/\partial\theta$ has all bounded and greater-than-zero eigenvalues, the expectation of $\partial^2 l/\partial\theta^2$ has all bounded and greater-than-zero eigenvalues. The density is finite and away from zero when $x_1$ is near the decision boundary. In addition, all the eigenvalues of 
    \begin{eqnarray*}
    \mathbb{E}_{X_1}\|X_1\|^2\left[ \left(\frac{\partial p(X_1;\beta^*)}{\partial\beta}\right)^{\top}\mathbb{E}_{X_1}\left\{\left(\frac{\partial p(X_1;\beta^*)}{\partial\beta}\right)\left(\frac{\partial p(X_1;\beta^*)}{\partial\beta}\right)^{\top}\right\}^{-1}\left(\frac{\partial p(X_1;\beta^*)}{\partial\beta}\right) \right]^2(X_1^{\top}\theta^*)^2
    \end{eqnarray*}
    are in $\Theta(d_1^3)$.
\end{enumerate}
\end{assumption}

Conditions (3), (4), (5) in Assumption \ref{assumption:more} regulate the distributions in $\mathcal{P}$, ensuring that a second-order Taylor expansion is accurate for the likelihood/loss function.

Assumption \ref{assumption:extra} is an extra assumption for adversarial training. Since we use SSL in clean training and use the vanilla method in adversarial training, some extra conditions are needed to supplement adversarial training. Similar to the idea in Assumption \ref{assumption:more}, Assumption \ref{assumption:extra} ensures that the Taylor expansions w.r.t. $p$ and $l$ are accurate.

\begin{lemma}\label{lem:minimax:mu}
Assume $(x_1,x_2)\in\mathbb{R}^{d_1+d_2}$, and $(x_1,x_2)\sim N(\mu,\Sigma^*)$ for some known $\Sigma^*$ and unknown $\mu=(\mu_1,\mu_2)$. There are $n_1+n_3$ samples of $(x_1,x_2)$, $n_2+n_4$ samples of $(x_1)$. Then for estimators of $\mu_1$,
\begin{eqnarray*}
\inf_{\widehat{\mu}_1}\sup_{\|\mu_1\|\leq R,\Sigma^*}\mathbb{E}\|\widehat{\mu}_1-\mu_1\|^2 = \Omega\left( \frac{R^2 d_1 }{\sum n_i} \right),
\end{eqnarray*}
and for estimators of $\mu_2$,
\begin{eqnarray*}
\inf_{\widehat{\mu}_2}\sup_{\|\mu_2\|\leq R,\Sigma^*}\mathbb{E}\|\widehat{\mu}_2-\mu_2\|^2 = \Omega\left( \frac{R^2 d_2 }{n_1+n_3} \right).
\end{eqnarray*}
\end{lemma}
\begin{proof}[Proof of Lemma \ref{lem:minimax:mu}]
The proof is similar to Lemma \ref{lem:minimax:a} below using Gaussian prior.
\end{proof}
\begin{lemma}\label{lem:minimax:a}
Assume $(x_1,x_2)\in\mathbb{R}^{d_1+d_2}$, and $(x_1,x_2)\sim N(\textbf{0},\Sigma^*)$ for some known $\Sigma^*$, and the response $\mathbb{E}[y|x_1]=x_1^{\top}\Sigma_{1,2}^*a$ for some vector $a$ and $Var(y|x_1)=\sigma^2$ for any $x_1$. There are $n_1$ samples of $(x_1,x_2,y)$, $n_2$ samples of $(x_1,y)$, $n_3$ samples of $(x_1,x_2)$ and $n_4$ samples of $(x_1)$. Then
\begin{eqnarray*}
\inf_{\widehat a}\sup_{a}\mathbb{E}\|\widehat a-a\|^2=\Theta\left(\frac{\sigma^2 d_2}{n_1+n_2}\right).
\end{eqnarray*}
\end{lemma}
\begin{proof}[Proof of Lemma \ref{lem:minimax:a}]
Assume $y|x_1$ follows Gaussian distribution. We take $a\sim N(0, \sigma^2 I_{d_2}/(\alpha n))$ for some $\alpha>0$ as the prior distribution of $a$. Then it is easy to see that only $S_1$ and $S_2$ are related to $a$. Denote $\widehat{\Sigma}_{n_1,n_2}$ as the sample covariance matrix, and $\widehat{a}_{\alpha}=(\widehat{\Sigma}_{n_1,n_2}+\alpha I_{d_1})^{-1}\frac{1}{n_1+n_2}\sum_{S_1,S_2}x_1 y$. The conditional distribution of $a|S_1,S_2$ becomes a multivariate Gaussian
\begin{eqnarray*}
a|S_1,S_2\sim N(\widehat{a}_{\alpha},(\sigma^2/(n_1+n_2))(\widehat{\Sigma}_{n_1,n_2}+\alpha  I_{d_2})^{-1}),
\end{eqnarray*}
and therefore
\begin{eqnarray*}
\inf_{\widehat a}\sup_{a}\mathbb{E}\|\widehat a-a\|^2\geq \inf_{\widehat a}\mathbb{E}\|\widehat{a}-\widehat{a}_{\alpha}\|^2+\|\widehat{a}_{\alpha}-a\|^2 =\Theta\left(\frac{\sigma^2 d_2}{n_1+n_2}\right).
\end{eqnarray*}
\end{proof}
\begin{lemma}\label{lem:minimax:covariance}
    Assume $(x_1,x_2)\in\mathbb{R}^{d_1+d_2}$, and $(x_1,x_2)\sim N(\mu,\Sigma^*)$, where $\Sigma^*=\begin{bmatrix}
    \Sigma_{1,1}^* & \Sigma_{1,2}^*\\
    \Sigma_{2,1}^* & \Sigma_{2,2}^*
    \end{bmatrix}$. Assume there are $n_1+n_3$ samples of $(x_1,x_2)$, and $n_2+n_4$ samples of $x_1$, and no response $y$ is provided, then when $(d_1+d_2)=o( (n_1+n_3)\log (n_1+n_3) )$,
    \begin{eqnarray*}\label{eqn:lem:minimax:covariance:1}
    \inf_{\widehat\Sigma_{1,2}}\sup_{\Sigma_{1,2}^*}\mathbb{E}\|\widehat\Sigma_{1,2}-\Sigma_{1,2}^*\|^2=\Omega\left(\frac{d_1+d_2}{n_1+n_3}\right).
    \end{eqnarray*}
    Assume the response $\mathbb{E}[y|x_1]=x_1^{\top}\Sigma_{1,2}^*a$ for some known vector $a$ with $\|a\|=1$ and $\mu=0$. There are $n_1$ samples of $(x_1,x_2,y)$, $n_2$ samples of $(x_1,y)$, $n_3$ samples of $(x_1,x_2)$ and $n_4$ samples of $(x_1)$, then when $(n_1+n_2)=o(n_1+n_3)$, for any estimator $\widehat\theta$ which estimates $\Sigma_{1,2}^*a$, 
    \begin{eqnarray*}
    \inf_{\widehat{\theta}}\sup_{ \Sigma_{1,2}^*}\mathbb{E}\|\widehat{\theta}-\Sigma_{1,2}^*a\|^2=\Omega\left(\frac{d_1+d_2}{n_1+n_3}\right).
    \end{eqnarray*}
\end{lemma}
\begin{proof}[Proof of Lemma \ref{lem:minimax:covariance}]
We directly prove the second argument of Lemma \ref{lem:minimax:covariance}. We use Bayes method to show the minimax lower bound. Assume $\Sigma_{1,2}=\Sigma_{1,2}^*$ follows some prior distribution, then
\begin{eqnarray*}
\inf_{\widehat\Sigma_{1,2}}\sup_{\Sigma_{1,2}^*}\mathbb{E}\|\widehat\Sigma_{1,2}-\Sigma_{1,2}^*\|^2&\geq& \inf_{\widehat\Sigma_{1,2}}\mathbb{E}_{S_i}\mathbb{E}_{\Sigma_{1,2}|S_i}\|\widehat\Sigma_{1,2}-\Sigma_{1,2}\|^2\\
&\geq&\inf_{\widehat\Sigma_{1,2}}\mathbb{E}_{S_i}\mathbb{E}_{\Sigma_{1,2}|S_i}\|[\mathbb{E}_{\Sigma_{1,2}|S_i}\Sigma_{1,2}]-\Sigma_{1,2}\|^2.
\end{eqnarray*}

 Denote the density of $\Sigma_{1,2}$ as $g(\Sigma_{1,2})$. Assume $\mathbb{E}\Sigma_{1,2}=\Sigma_{1,2}^*$ and $\Sigma_{1,2}-\Sigma_{1,2}^*=\Delta_1\Delta_2^{\top}$. And we take $\Delta_2=a/\|a\|$.

Assume $\Sigma_{1,1}=I_{d_1}$, $\Sigma_{1,2}=I_{d_2}$, and $\theta_0=\Sigma_{1,2}a$. Denote $\Sigma^*=\begin{bmatrix}
\Sigma_{1,1} & \Sigma_{1,2}^*\\
\Sigma_{2,1}^* & \Sigma_{2,2}
\end{bmatrix}$. Then the likelihood of the four types of samples $S_1,...,S_4$ is proportional to 
\begin{eqnarray*}
\frac{g(\Sigma_{1,2})}{|\Sigma|^{(n_1+n_3)/2}}\exp\left\{ -\frac{1}{2}\sum_{S_1,S_3}\begin{bmatrix} x_1^{\top},x_2^{\top}
\end{bmatrix}\Sigma^{-1}\begin{bmatrix} x_1\\x_2
\end{bmatrix} -\frac{1}{2\sigma^2}\sum_{S_1,S_2}(y-x_1^{\top}\Sigma_{1,2}a)^2 \right\}.
\end{eqnarray*}
Since $\Sigma_{1,1}$ and $\Sigma_{2,2}$ are both identity matrix, 
\begin{eqnarray*}
|\Sigma|&=& \left|\begin{bmatrix}
\Sigma_{1,1} & \Sigma_{1,2}^*\\
\Sigma_{2,1}^*+\Delta_2\Delta_1^{\top} & \Sigma_{2,2}
\end{bmatrix}\right|\left(1+\begin{bmatrix}\textbf{0}& \Delta_2^{\top}
\end{bmatrix}\begin{bmatrix}
\Sigma_{1,1} & \Sigma_{1,2}^*\\
\Sigma_{2,1}^*+\Delta_2\Delta_1^{\top} & \Sigma_{2,2}
\end{bmatrix}^{-1}
\begin{bmatrix}\Delta_1 \\ \textbf{0}
\end{bmatrix}\right)\\
&=&|\Sigma^*|\left(1+\begin{bmatrix}\textbf{0}& \Delta_2^{\top}
\end{bmatrix}(\Sigma^*)^{-1}
\begin{bmatrix}\Delta_1 \\ \textbf{0}
\end{bmatrix}\right)\left(1+\begin{bmatrix}\textbf{0}& \Delta_2^{\top}
\end{bmatrix}\begin{bmatrix}
\Sigma_{1,1} & \Sigma_{1,2}^*\\
\Sigma_{2,1}^*+\Delta_2\Delta_1^{\top} & \Sigma_{2,2}
\end{bmatrix}^{-1}
\begin{bmatrix}\Delta_1 \\ \textbf{0}
\end{bmatrix}\right),
\end{eqnarray*}
where
\begin{eqnarray*}
\begin{bmatrix}
\Sigma_{1,1} & \Sigma_{1,2}^*\\
\Sigma_{2,1}^*+\Delta_2\Delta_1^{\top} & \Sigma_{2,2}
\end{bmatrix}^{-1}
=(\Sigma^*)^{-1}-\frac{(\Sigma^*)^{-1}\begin{bmatrix}
\Delta_1\\0
\end{bmatrix}\begin{bmatrix}\textbf{0}& \Delta_2^{\top}
\end{bmatrix}(\Sigma^*)^{-1}}{1+\begin{bmatrix}\textbf{0}& \Delta_2^{\top}
\end{bmatrix}(\Sigma^*)^{-1}\begin{bmatrix}
\Delta_1\\0
\end{bmatrix}},
\end{eqnarray*}
thus denoting $\xi=\begin{bmatrix}\textbf{0}& \Delta_2^{\top}
\end{bmatrix}(\Sigma^*)^{-1}
\begin{bmatrix}\Delta_1 \\ \textbf{0}
\end{bmatrix}$, we obtain
\begin{eqnarray*}
|\Sigma|&=&|\Sigma^*|\left( 1+\xi\right)\left( 1+ \begin{bmatrix}\textbf{0}& \Delta_2^{\top}
\end{bmatrix}\left[(\Sigma^*)^{-1}-\frac{(\Sigma^*)^{-1}\begin{bmatrix}
\Delta_1\\0
\end{bmatrix}\begin{bmatrix}\textbf{0}& \Delta_2^{\top}
\end{bmatrix}(\Sigma^*)^{-1}}{1+\xi}\right]
\begin{bmatrix}\Delta_1 \\ \textbf{0}
\end{bmatrix}\right)\\
&=&|\Sigma^*|\left( 1+\xi\right)\left( 1+ \xi - \frac{\xi^2}{1+\xi}\right)\\
&=&|\Sigma^*|(1+2\xi).
\end{eqnarray*}
In terms of $\Sigma^{-1}$, denoting  $\xi_1=\begin{bmatrix} \Delta_1^{\top} & \textbf{0}
\end{bmatrix}(\Sigma^*)^{-1}
\begin{bmatrix}\Delta_1 \\ \textbf{0}
\end{bmatrix}$, and  $\xi_2=\begin{bmatrix}\textbf{0}& \Delta_2^{\top}
\end{bmatrix}(\Sigma^*)^{-1}
\begin{bmatrix}\textbf{0}\\\Delta_2 
\end{bmatrix}$
\begin{eqnarray*}
\Sigma^{-1} &=& \begin{bmatrix}
\Sigma_{1,1} & \Sigma_{1,2}^*\\
\Sigma_{2,1}^*+\Delta_2\Delta_1^{\top} & \Sigma_{2,2}
\end{bmatrix}^{-1}-\frac{ \begin{bmatrix}
\Sigma_{1,1} & \Sigma_{1,2}^*\\
\Sigma_{2,1}^*+\Delta_2\Delta_1^{\top} & \Sigma_{2,2}
\end{bmatrix}^{-1} \begin{bmatrix}
\textbf{0}\\\Delta_2
\end{bmatrix}\begin{bmatrix}\Delta_1^{\top} & \textbf{0}
\end{bmatrix}\begin{bmatrix}
\Sigma_{1,1} & \Sigma_{1,2}^*\\
\Sigma_{2,1}^*+\Delta_2\Delta_1^{\top} & \Sigma_{2,2}
\end{bmatrix}^{-1}  }{1+\xi-\xi^2/(1+\xi)}\\
&=&(\Sigma^*)^{-1}-\frac{(\Sigma^*)^{-1}\begin{bmatrix}
\Delta_1\\\textbf{0}
\end{bmatrix}\begin{bmatrix}\textbf{0}& \Delta_2^{\top}
\end{bmatrix}(\Sigma^*)^{-1}}{1+\xi}\\
&& - \frac{ (\Sigma^*)^{-1} \begin{bmatrix}
\textbf{0}\\\Delta_2
\end{bmatrix}\begin{bmatrix} \Delta_1^{\top} & \textbf{0}
\end{bmatrix}(\Sigma^*)^{-1}  }{1+\xi-\xi^2/(1+\xi)}-\frac{\xi_1\xi_2 (\Sigma^*)^{-1} \begin{bmatrix}
\Delta_1\\\textbf{0}
\end{bmatrix}\begin{bmatrix}  \textbf{0} & \Delta_2^{\top} 
\end{bmatrix}(\Sigma^*)^{-1} }{(1+\xi)^2(1+\xi-\xi^2/(1+\xi))}\\
&&+\frac{ \xi_2 (\Sigma^*)^{-1}\begin{bmatrix}
\Delta_1\\\textbf{0}
\end{bmatrix} \begin{bmatrix}
\Delta_1& \textbf{0}
\end{bmatrix}(\Sigma^*)^{-1}}{(1+\xi)(1+\xi-\xi^2/(1+\xi))}+\frac{ \xi_1 (\Sigma^*)^{-1}\begin{bmatrix}
\textbf{0}\\\Delta_2
\end{bmatrix} \begin{bmatrix}
\textbf{0} & \Delta_2
\end{bmatrix}(\Sigma^*)^{-1}}{(1+\xi)(1+\xi-\xi^2/(1+\xi))}.
\end{eqnarray*}
Therefore, 
\begin{eqnarray*}
&&\frac{g(\Sigma_{1,2})}{|\Sigma|^{(n_1+n_3)/2}}\exp\left\{ -\frac{1}{2}\sum_{S_1,S_3}\begin{bmatrix} x_1^{\top},x_2^{\top}
\end{bmatrix}\Sigma^{-1}\begin{bmatrix} x_1\\x_2
\end{bmatrix} -\frac{1}{2\sigma^2}\sum_{S_1,S_2}(y-x_1^{\top}\Sigma_{1,2}a)^2 \right\}\\
&=&
\frac{g(\Sigma_{1,2})}{|\Sigma^*|^{(n_1+n_3)/2}}\exp\left\{ -\frac{1}{2}\sum_{S_1,S_3}\begin{bmatrix} x_1^{\top},x_2^{\top}
\end{bmatrix}(\Sigma^*)^{-1}\begin{bmatrix} x_1\\x_2
\end{bmatrix} -\frac{1}{2\sigma^2}\sum_{S_1,S_2}(y-x_1^{\top}\Sigma_{1,2}^*a)^2 \right\}\\
&&\times \frac{1}{(1+2\xi)^{(n_1+n_3)/2}}\exp\left\{ -\frac{1}{2}\sum_{S_1,S_3} \begin{bmatrix} x_1^{\top},x_2^{\top}
\end{bmatrix}(\Sigma^{-1}-(\Sigma^*)^{-1})\begin{bmatrix} x_1\\x_2
\end{bmatrix} \right\}\\
&&\times \exp\left\{-\frac{1}{2\sigma^2}\sum_{S_1,S_2}(x_1^{\top}\Delta_1\Delta_2^{\top}a)^2-2 (y-x_1^{\top}\Sigma_{1,2}^*a)(x_1^{\top}\Delta_1\Delta_2^{\top}a)\right\}\\
&:=& g(\Sigma_{1,2}) g_0 g_1(\Delta_1,\Delta_2) g_2(\Delta_1,\Delta_2)
\end{eqnarray*}
Denoting $\begin{bmatrix}
z_1\\z_2
\end{bmatrix}=(\Sigma^*)^{-1}\begin{bmatrix}
x_1\\x_2
\end{bmatrix}$, then
\begin{eqnarray*}
&& -\frac{1}{2}\sum_{S_1,S_3} \begin{bmatrix} x_1^{\top},x_2^{\top}
\end{bmatrix}(\Sigma^{-1}-(\Sigma^*)^{-1})\begin{bmatrix} x_1\\x_2
\end{bmatrix}\\
&=&-\frac{1}{2}\sum_{S_1,S_3} -\frac{ [z_1^{\top},z_2^{\top}]\begin{bmatrix}
\Delta_1\\\textbf{0}
\end{bmatrix}[z_1^{\top},z_2^{\top}]\begin{bmatrix}
\textbf{0}\\\Delta_2
\end{bmatrix} }{1+\xi} - \frac{  [z_1^{\top},z_2^{\top}]\begin{bmatrix}
\Delta_1\\\textbf{0}
\end{bmatrix}[z_1^{\top},z_2^{\top}]\begin{bmatrix}
\textbf{0}\\\Delta_2
\end{bmatrix} }{1+\xi-\xi^2/(1+\xi)}- \frac{\xi_1\xi_2 (\Sigma^*)^{-1} \begin{bmatrix}
\Delta_1\\\textbf{0}
\end{bmatrix}\begin{bmatrix}  \textbf{0} & \Delta_2^{\top} 
\end{bmatrix}(\Sigma^*)^{-1} }{(1+\xi)^2(1+\xi-\xi^2/(1+\xi))}\\
&&-\frac{1}{2}\sum_{S_1,S_3}\frac{\xi_2 \left([z_1^{\top},z_2^{\top}]\begin{bmatrix}
\Delta_1\\\textbf{0}
\end{bmatrix}\right)^2}{(1+\xi)(1+\xi-\xi^2/(1+\xi))}+\frac{\xi_1 \left([z_1^{\top},z_2^{\top}]\begin{bmatrix}
\textbf{0}\\\Delta_2
\end{bmatrix}\right)^2}{(1+\xi)(1+\xi-\xi^2/(1+\xi))}.
\end{eqnarray*}

Now we consider another likelihood
\begin{eqnarray*}
g(\Sigma_{1,2}) g_0 \widetilde{g}_1(\Delta_1,\Delta_2),
\end{eqnarray*}
where $\widetilde{g}_1$ is an approximation of $g_1$ and equals to
\begin{eqnarray*}
\widetilde{g}_1(\Delta_1,\Delta_2)
&=&\exp\left\{ tr\left(\begin{bmatrix}
\Delta_1\\\textbf{0}
\end{bmatrix} \begin{bmatrix}
\textbf{0} & \Delta_2
\end{bmatrix}(\Sigma^*)^{-1}\left[\sum_{S_1,S_3}(\Sigma^*)^{-1}\begin{bmatrix}
x_1\\x_2
\end{bmatrix} \begin{bmatrix}
x_1^{\top},x_2^{\top}
\end{bmatrix}-I_{d_1+d_2}\right] \right) \right\}\\
&&\times\exp\left\{ -\frac{1}{2}\sum_{S_1,S_3}\xi_2 \left([z_1^{\top},z_2^{\top}]\begin{bmatrix}
\Delta_1\\\textbf{0}
\end{bmatrix}\right)^2+\xi_1 \left([z_1^{\top},z_2^{\top}]\begin{bmatrix}
\textbf{0}\\\Delta_2
\end{bmatrix}\right)^2 \right\}\\
&&\times 1\left\{\Delta_1\in\mathcal{S}\right\}\\
&=&\exp\left\{ tr\left(\begin{bmatrix}
\Delta_1\\\textbf{0}
\end{bmatrix} \begin{bmatrix}
\textbf{0} & \Delta_2
\end{bmatrix}(\Sigma^*)^{-1}\left[\sum_{S_1,S_3}(\Sigma^*)^{-1}\begin{bmatrix}
x_1\\x_2
\end{bmatrix} \begin{bmatrix}
x_1^{\top},x_2^{\top}
\end{bmatrix}-I_{d_1+d_2}\right] \right) \right\}\\
&&\times\exp\left\{ -\frac{1}{2} \begin{bmatrix}
\Delta_1^{\top} & \textbf{0}
\end{bmatrix}\left( \sum_{S_1,S_3}\xi_2\begin{bmatrix}
x_1\\x_2
\end{bmatrix}(\Sigma^*)^{-2}[x_1^{\top},x_2^{\top}] +(x_1(\Sigma^*)^{-1}\Delta_2)^2I_{d_1} \right)\begin{bmatrix}
\Delta_1\\\textbf{0}
\end{bmatrix} \right\}\\
&&\times 1\left\{\Delta_1\in\mathcal{S}\right\}.
\end{eqnarray*}

Intuitively, when $\|\Delta_1\|\rightarrow 0$, $g_1\rightarrow \widetilde{g}_1$; otherwise $g_1\rightarrow 0$. As a result, we take $\mathcal{S}=\{\|\Delta_{1}\|_{\infty}\leq 1/\sqrt{n_1+n_3}\}$.

From the generation of $(x_1,x_2)$, assume $g_0(\Sigma_{1,2})\propto 1\{\|\Delta_{1}\|_{\infty}\leq 1/\sqrt{n_1+n_3}\}$. We know that with probability tending to 1,
\begin{eqnarray*}
\left\|a\left(\sum_{S_1,S_3}(\Sigma^*)^{-1}\begin{bmatrix}
x_1\\x_2
\end{bmatrix} \begin{bmatrix}
x_1^{\top},x_2^{\top}
\end{bmatrix}-I_{d_1+d_2}\right)\right\|_2^2 =O\left((d_1+d_2)(n_1+n_3)\right),
\end{eqnarray*}
and
\begin{eqnarray*}
\left( \sum_{S_1,S_3}\xi_2\begin{bmatrix}
x_1\\x_2
\end{bmatrix}(\Sigma^*)^{-2}[x_1^{\top},x_2^{\top}] +(x_1(\Sigma^*)^{-1}\Delta_2)^2I_{d_1} \right)=O\left( (n_1+n_3)I_{d_1} \right).
\end{eqnarray*}
Therefore, $\Delta_1$ follows truncated normal distribution and
\begin{eqnarray*}
\mathbb{E}_{\Delta_1\sim \widetilde{g}_1}\|(\Delta_1\Delta_2^{\top}-\mathbb{E}\Delta_1\Delta_2^{\top})a\|^2 =\Theta\left(\frac{d_1+d_2}{n_1+n_3}\right).
\end{eqnarray*}

In the above analysis of $A$ we investigate in the distribution $gg_0\widetilde{g}_1$ instead of the true distribution $g g_0g_1g_2$. Now we quantify the difference between $\widetilde{g}_1$ and $g_1g_2$.

When $\Delta_1\in\mathcal{S}$, we have
\begin{eqnarray*}
&&\frac{g_1g_2}{\widetilde{g}_1}\\&=&\exp\bigg\{\frac{1}{2} tr\left(\begin{bmatrix}
\Delta_1\\\textbf{0}
\end{bmatrix} \begin{bmatrix}
\textbf{0} & \Delta_2
\end{bmatrix}(\Sigma^*)^{-1}\left[\sum_{S_1,S_3}(\Sigma^*)^{-1}\begin{bmatrix}
x_1\\x_2
\end{bmatrix} \begin{bmatrix}
x_1^{\top},x_2^{\top}
\end{bmatrix}-I_{d_1+d_2}\right] \right)\\&&\qquad\qquad\qquad\qquad\qquad\qquad\qquad\qquad\qquad\qquad\times\left(\frac{1}{1+\xi}+\frac{1}{1+\xi-\xi^2/(1+\xi)}-2\right) \bigg\}\\
&&\times\exp\bigg\{-\frac{1}{2}\sum_{S_1,S_3}\frac{\xi_2 \left([z_1^{\top},z_2^{\top}]\begin{bmatrix}
\Delta_1\\\textbf{0}
\end{bmatrix}\right)^2}{(1+\xi)(1+\xi-\xi^2/(1+\xi))}\\&&\qquad\qquad\qquad\qquad\qquad\qquad\qquad\qquad+\frac{\xi_1 \left([z_1^{\top},z_2^{\top}]\begin{bmatrix}
\textbf{0}\\\Delta_2
\end{bmatrix}\right)^2}{(1+\xi)(1+\xi-\xi^2/(1+\xi))}\left(\frac{2}{(1+\xi)(1+\xi-\xi^2/(1+\xi))}-2\right) \bigg\}\\
&&\times\exp\left\{-\frac{1}{2}\sum_{S_1,S_3}\frac{\xi_1\xi_2 (\Sigma^*)^{-1} \begin{bmatrix}
\Delta_1\\\textbf{0}
\end{bmatrix}\begin{bmatrix} \textbf{0} & \Delta_2^{\top}
\end{bmatrix}(\Sigma^*)^{-1} }{(1+\xi)^2(1+\xi-\xi^2/(1+\xi))}\right\}\\
&&\times \exp\left\{-\frac{1}{2\sigma^2}\sum_{S_1,S_2}(x_1^{\top}\Delta_1\Delta_2^{\top}a)^2-2 (y-x_1^{\top}\Sigma_{1,2}^*a)(x_1^{\top}\Delta_1\Delta_2^{\top}a)\right\}\\
&&\times 1\left\{\Delta_1\in\mathcal{S}\right\}\\
&:=&A\times B\times C\times D\times  1\left\{\Delta_1\in\mathcal{S}\right\}.
\end{eqnarray*}
Recall that $\xi=\begin{bmatrix}\textbf{0}& \Delta_2^{\top}
\end{bmatrix}(\Sigma^*)^{-1}
\begin{bmatrix}\Delta_1 \\ \textbf{0}
\end{bmatrix}$ and $\Delta_2=a/\|a\|$, from the support of $\Delta_1$, we have $\xi=O(\sqrt{d_1/(n_1+n_3)})$. Therefore,
\begin{eqnarray*}
&&\frac{1}{1+\xi}+\frac{1}{1+\xi-\xi^2/(1+\xi)}-2\\
&=&\frac{1}{(1+\xi)[1+\xi-\xi^2/(1+\xi)]}\left[1+\xi-\xi^2/(1+\xi)+1+\xi-2 (1+\xi)[1+\xi-\xi^2/(1+\xi)]\right]\\
&=&\frac{1}{(1+\xi)[1+\xi-\xi^2/(1+\xi)]}\left[1+\xi-\xi^2/(1+\xi)+1+\xi-  2 [1+\xi-\xi^2/(1+\xi)]\right]\\
&&-\frac{1}{(1+\xi)[1+\xi-\xi^2/(1+\xi)]}2 [\xi+\xi^2-\xi^3/(1+\xi)]\\
&=&\frac{1}{(1+\xi)[1+\xi-\xi^2/(1+\xi)]}\left[\xi^2/(1+\xi)-2 [\xi+\xi^2-\xi^3/(1+\xi)\right]\\
&=&O(\xi),
\end{eqnarray*}
so $A=\exp\{o(1)\}$. Similarly, we have $B=\exp\{o(1)\}$. In terms of $C$ and $D$, one can show that both of them are in $\exp\{o(1)\}$ as well. Therefore, $g_1g_2=\widetilde{g}_1(1+o(1))$, which implies that 
\begin{eqnarray*}
\sup_{\Sigma}\mathbb{E}\|(\Delta_1\Delta_2^{\top}-\mathbb{E}\Delta_1\Delta_2^{\top})a\|^2 = \Theta\left(\frac{d_1+d_2}{n_1+n_3}\right).
\end{eqnarray*}

As a result, we can conclude that 
\begin{eqnarray*}
\inf_{\widehat\theta}\sup_{\Sigma}\mathbb{E}\|\widehat\theta- \Sigma_{1,2}a \|^2=\Theta\left(\frac{d_1+d_2}{n_1+n_3}\right).
\end{eqnarray*}


\end{proof}

\section{Proofs for Section 2 and 3}
\subsection{Proof for Theorem 1}
\begin{proof}[Proof of Theorem \ref{thm:mixed:lower:no_ci}, CI holds]
The basic idea is similar to Theorem \ref{thm:linear}. We impose a prior distribution on the parameter $\mu_0$ associated with $P(y|x_1)$, then argue that we cannot exactly estimate $\mu_0$.

Denote $f_1$ and $f_2$ are the density of $
x_1$ and $x_2$, and $x_2|y\sim N(y\mu_2,\Sigma_{2,2})$ and is conditionally independent to $x_1$ given $y$. Also denote $\psi(x,\mu,\Sigma)$ as normal density of $x$ given mean $\mu$ and covariance $\Sigma$. Assume $P(y=1|x_1)=\psi(x_1,\mu_0,\Sigma_{1,1})/(\psi(x_1,\mu_0,\Sigma_{1,1})+\psi(x_1,-\mu_0,\Sigma_{1,1}))$, then the likelihood becomes
\begin{eqnarray*}
&&f_0(\mu_0)\left(\prod_{S_1}f_1(x_1)P(y|x_1,\mu_0)f_2(x_2|y)\right)\times\left(\prod_{S_2}f_1(x_1)P(y|x_1,\mu_0)\right)\\&&
\times\left(\prod_{S_3}f_1(x_1)f_2(x_2|x_1,\mu_0)\right)\times\left(\prod_{S_4}f_1(x_1)\right).
\end{eqnarray*}
Thus the posterior distribution of $\mu_0$ given $S_1$ to $S_4$ is proportional to
\begin{eqnarray*}
f_0(\mu_0)\left(\prod_{S_1,S_2}P(y|x_1,\mu_0)\right)\times\left(\prod_{S_3}f_2(x_2|x_1,\mu_0)\right).
\end{eqnarray*}

Denote $f_2(x_2|x_1,\mu_0)=m(\mu_0,x_1,x_2)$, then we have
\begin{eqnarray*}
\frac{\partial}{\partial{\mu_0}}\log m(\mu_0^*,x_1,x_2)
&=&\frac{\psi\left([x_1,x_2],[\mu_0^*,\mu_2],\Sigma\right)-\psi\left([x_1,x_2],-[\mu_0^*,\mu_2],\Sigma\right) }{\psi\left([x_1,x_2],[\mu_0^*,\mu_2],\Sigma\right)+\psi\left([x_1,x_2],-[\mu_0^*,\mu_2],\Sigma\right)}(x_1-\mu_0^*)\\
&&-\frac{\psi(x_1,\mu_0^*,\Sigma_{1,1})-\psi(x_1,-\mu_0^*,\Sigma_{1,1})  }{\psi(x_1,\mu_0^*,\Sigma_{1,1})+\psi(x_1,-\mu_0^*,\Sigma_{1,1})}(x_1-\mu_0^*),
\end{eqnarray*}
and
\begin{eqnarray*}
\frac{\partial^2}{\partial \mu_0^2}\log m(\mu_0^*,x_1,x_2)
&=&-\frac{\psi\left([x_1,x_2],[\mu_0^*,\mu_2],\Sigma\right)-\psi\left([x_1,x_2],-[\mu_0^*,\mu_2],\Sigma\right) }{\psi\left([x_1,x_2],[\mu_0^*,\mu_2],\Sigma\right)+\psi\left([x_1,x_2],-[\mu_0^*,\mu_2],\Sigma\right)} I_{d_1}\\
&&+\frac{\psi(x_1,\mu_0^*,\Sigma_{1,1})-\psi(x_1,-\mu_0^*,\Sigma_{1,1})  }{\psi(x_1,\mu_0^*,\Sigma_{1,1})+\psi(x_1,-\mu_0^*,\Sigma_{1,1})}I_{d_1}\\
&&-\left(\frac{\psi\left([x_1,x_2],[\mu_0^*,\mu_2],\Sigma\right)-\psi\left([x_1,x_2],-[\mu_0^*,\mu_2],\Sigma\right) }{\psi\left([x_1,x_2],[\mu_0^*,\mu_2],\Sigma\right)+\psi\left([x_1,x_2],-[\mu_0^*,\mu_2],\Sigma\right)}\right)^2(x_1-\mu_0^*)(x_1-\mu_0^*)^{\top}\\
&&+\left(\frac{\psi(x_1,\mu_0^*,\Sigma_{1,1})-\psi(x_1,-\mu_0^*,\Sigma_{1,1})  }{\psi(x_1,\mu_0^*,\Sigma_{1,1})+\psi(x_1,-\mu_0^*,\Sigma_{1,1})}\right)^2(x_1-\mu_0^*)(x_1-\mu_0^*)^{\top}.
\end{eqnarray*}
When $(d_1+d_2)^2\log(n_1+n_3)=o(n_1+n_3)$, one can figure out that $\frac{\partial}{\partial \mu_0}\log m(\mu_0^*,x_1,x_2)$ and $\frac{\partial^2}{\partial \mu_0^2}\log m(\mu_0^*,x_1,x_2)$ converges to their mean respectively. In addition, from the model construction, we know that $\mathbb{E}\frac{\partial^2}{\partial \mu_0^2}\log m(\mu_0^*,x_1,x_2)$ is negatively definite. 

Using Taylor expansion, we have
\begin{eqnarray*}
\log m(\mu_0,x_1,x_2)&=&\log m(\mu_0^*,x_1,x_2)+\left(\frac{\partial}{\partial \mu_0}\log m(\mu_0^*,x_1,x_2)\right)^{\top}(\mu_0-\mu_0^*)\\
&&+\frac{1}{2}(\mu_0-\mu_0^*)^{\top}\left(\frac{\partial^2}{\partial \mu_0^2}\log m(\mu_0^*,x_1,x_2)\right)(\mu_0-\mu_0^*)+O(\|\mu_0-\mu_0^*\|^3).
\end{eqnarray*}
Thus when $\|\mu_0-\mu_0^*\|^3=O(1/(n_1+n_3))$, $\prod_{S_3}f_2$ only changes in $O(1)$ proportion.

On the other hand, one can figure out that when $\|\mu_0-\mu_0^*\|^3=O( 1/(n_1+n_2) )$, the posterior distribution of $\mu_0$ is approximately a multivariate Gaussian. The final lower bound takes from the smaller one in the above two bounds.

When $\epsilon>0$, assume $x_1\sim N(\textbf{0},I_{d_1})$, then the adversarial risk minimizer still the linear classifier $\sgn(x_1^{\top}\mu_0)$. As a result, the minimax lower bound of the estimation error of $\mu_0$ is inherited in adversarial case.  
\end{proof}

\begin{proof}[Proof of Theorem \ref{thm:mixed:lower:no_ci}, CI does not hold]
We consider several cases:
\begin{itemize}
    \item Case 1: $n_1+n_2=o(n_1+n_3)$ and $d_2\leq d_1$.
    \item Case 2: $n_1+n_2=o(n_1+n_3)$ and $d_2> d_1$.
    \item Case 3: $n_1+n_3=O(n_1+n_2)$.
\end{itemize}
\textit{Case 1}: The proof is similar to Theorem \ref{thm:regression} for regression. Assume the optimal classifier w.r.t $x_1$ is of the form $\sgn(a^{\top}\Sigma_{2,1}x_1)$ for some $a\in\mathbb{R}^{d_1}$. Based on Lemma \ref{lem:minimax:covariance}, when $a$ is known and $\|a\|=1$, we have 
\begin{eqnarray*}
\inf_{\widehat{\theta}}\sup_{\Sigma_{1,2}}\mathbb{E}\|\widehat{\theta}-\Sigma_{1,2}a\|^2=\Omega\left(\frac{d_1+d_2}{n_1+n_3}\right).
\end{eqnarray*}
When $\Sigma_{1,2}$ is known, the proof follows similar arguments as Theorem \ref{thm:mixed:lower:no_ci} when CI holds. Assume $P(y=1|x_1)=\psi(x_1,\Sigma_{1,2}a,\Sigma_{1,1})/(\psi(x_1,\Sigma_{1,2}a,\Sigma_{1,1})+\psi(x_1,-\Sigma_{1,2}a,\Sigma_{1,1}))$. Since there is no CI condition, there is no connection between the distribution of $x_2$ and the label $y$. As a result, the part of likelihood related to $a$ is
\begin{eqnarray*}
f_0(\Sigma_{1,2}a)\prod_{S_1,S_2}P(y|x_1,\Sigma_{1,2}a).
\end{eqnarray*}
Denote $P(y=1|x_1)=m(\Sigma_{1,2}a,x_1)$, then we have
\begin{eqnarray*}
\frac{\partial}{\partial \Sigma_{1,2}a}\log m(\Sigma_{1,2}a^*,x_1)&=& (x_1-\Sigma_{1,2}a^*)\\
&&-\frac{\psi(x_1,\Sigma_{1,2}a,\Sigma_{1,1})-\psi(x_1,-\Sigma_{1,2}a,\Sigma_{1,1}) }{ \psi(x_1,\Sigma_{1,2}a,\Sigma_{1,1})+\psi(x_1,-\Sigma_{1,2}a,\Sigma_{1,1})}(x_1-\Sigma_{1,2}a^*)\\
\frac{\partial}{\partial \Sigma_{1,2}a}\log [1-m(\Sigma_{1,2}a^*,x_1)]&=& -(x_1-\Sigma_{1,2}a^*)\\
&&-\frac{\psi(x_1,\Sigma_{1,2}a,\Sigma_{1,1})-\psi(x_1,-\Sigma_{1,2}a,\Sigma_{1,1}) }{ \psi(x_1,\Sigma_{1,2}a,\Sigma_{1,1})+\psi(x_1,-\Sigma_{1,2}a,\Sigma_{1,1})}(x_1-\Sigma_{1,2}a^*)
\end{eqnarray*}
\begin{eqnarray*}
&&\frac{\partial^2}{\partial \Sigma_{1,2}a^2}\log m(\Sigma_{1,2}a^*,x_1)\\
&=& -I_{d_1}-(x_1-\Sigma_{1,2}a^*)(x_1-\Sigma_{1,2}a^*)^{\top}\\
&&+\frac{\psi(x_1,\Sigma_{1,2}a^*,\Sigma_{1,1})-\psi(x_1,-\Sigma_{1,2}a^*,\Sigma_{1,1}) }{ \psi(x_1,\Sigma_{1,2}a,^*\Sigma_{1,1})+\psi(x_1,-\Sigma_{1,2}a^*,\Sigma_{1,1})}I_{d_1}\\
&&+\left(\frac{\psi(x_1,\Sigma_{1,2}a^*,\Sigma_{1,1})-\psi(x_1,-\Sigma_{1,2}a^*,\Sigma_{1,1}) }{ \psi(x_1,\Sigma_{1,2}a^*,\Sigma_{1,1})+\psi(x_1,-\Sigma_{1,2}a^*,\Sigma_{1,1})}\right)^2(x_1-\Sigma_{1,2}a^*)(x_1-\Sigma_{1,2}a^*)^{\top}\\
&&\frac{\partial^2}{\partial \Sigma_{1,2}(a^*)^2}\log[1- m(\Sigma_{1,2}a^*,x_1)]\\
&=& I_{d_1}-(x_1-\Sigma_{1,2}a^*)(x_1-\Sigma_{1,2}a^*)^{\top}\\
&&+\frac{\psi(x_1,\Sigma_{1,2}a^*,\Sigma_{1,1})-\psi(x_1,-\Sigma_{1,2}a^*,\Sigma_{1,1}) }{ \psi(x_1,\Sigma_{1,2}a^*,\Sigma_{1,1})+\psi(x_1,-\Sigma_{1,2}a^*,\Sigma_{1,1})}I_{d_1}\\
&&+\left(\frac{\psi(x_1,\Sigma_{1,2}a^*,\Sigma_{1,1})-\psi(x_1,-\Sigma_{1,2}a^*,\Sigma_{1,1}) }{ \psi(x_1,\Sigma_{1,2}a^*,\Sigma_{1,1})+\psi(x_1,-\Sigma_{1,2}a^*,\Sigma_{1,1})}\right)^2(x_1-\Sigma_{1,2}a^*)(x_1-\Sigma_{1,2}a^*)^{\top}.
\end{eqnarray*}
As a result, when the singular values of $\Sigma_{1,2}$ are all finite positive constants, if $\|a-a^*\|^3=O(1/(n_1+n_2))$, the posterior distribution of $a$ is approximately a multivariate Gaussian when $d_1>d_2$, and we obtain $\mathbb{E}\|a-a^*\|^2=\Theta(d_2/(n_1+n_2))$.  The overall rate becomes $\Theta(d_2/(n_1+n_2)+(d_1+d_2)/(n_1+n_3))$.

\textit{Case 2}: the arguments are similar to \textit{Case 1}. However, in the last step, the covariance of the posterior of $a$ if not of full rank, so again we obtain $\mathbb{E}\|a-a^*\|^2=\Theta(d_1/(n_1+n_2))$. On the other hand, since $n_1+n_2=o(n_1+n_3)$, the overall minimax rate becomes $\Theta(d_1/(n_1+n_2))$.

\textit{Case 3}: we directly assume a prior distribution on $\mu_0$ and do not consider the relationship between $\mu_0$ and $\Sigma_{1,2}$. Following similar arguments as in the previous cases, we obtain $\mathbb{E}\|\mu_0-\mu_0^*\|^2=\Theta(d_1/(n_1+n_2))$.

\end{proof}


\subsection{Proof of Example \ref{example:mixed_gaussian}}

\begin{proof}[Proof of Bahadur Representation in Example \ref{example:mixed_gaussian}]
In pretext task, we have
\begin{eqnarray*}
\frac{\partial}{\partial\Sigma_{1,1}^{-1}\mu_1}\| x_2-\phi(x_1) \|_2^2=-4(x_2-\phi(x_1))^{\top}\mu_2 \frac{\partial p(x_1;\Sigma_{1,1}^{-1}\mu_1)}{\partial \Sigma_{1,1}^{-1}\mu_1},
\end{eqnarray*}
\begin{eqnarray*}
\frac{\partial}{\partial\mu_2}\| x_2-\phi(x_1) \|_2^2=-2(2p(x_1)-1)(x_2-\phi(x_1)),
\end{eqnarray*}
and
\begin{eqnarray*}
\frac{\partial^2}{\partial (\Sigma_{1,1}^{-1}\mu_1)^2}\| x_2-\phi(x_1) \|_2^2&=&  -4(x_2-\phi(x_1))^{\top}\mu_2 \frac{\partial^2 p(x_1;\Sigma_{1,1}^{-1}\mu_1)}{\partial (\Sigma_{1,1}^{-1}\mu_1)^2}\\&&
+8\|\mu_2\|^2\left(\frac{\partial p(x_1;\Sigma_{1,1}^{-1}\mu_1)}{\partial (\Sigma_{1,1}^{-1}\mu_1)}\right)\left(\frac{\partial p(x_1;\Sigma_{1,1}^{-1}\mu_1)}{\partial (\Sigma_{1,1}^{-1}\mu_1)}\right)^{\top},
\end{eqnarray*}
\begin{eqnarray*}
\frac{\partial^2}{\partial \mu_2^2}\| x_2-\phi(x_1) \|_2^2=2(2p(x_1)-1)^2I_{d_2},
\end{eqnarray*}
\begin{eqnarray*}
\frac{\partial^2}{\partial (\Sigma_{1,1}^{-1}\mu_1)\partial\mu_2}\| x_2-\phi(x_1) \|_2^2=-4\frac{\partial p(x_1;\Sigma_{1,1}^{-1}\mu_1)}{\partial (\Sigma_{1,1}^{-1}\mu_1)}(x_2-\phi(x_1))^{\top} + 4(2p(x_1)-1)\frac{\partial p(x_1;\Sigma_{1,1}^{-1}\mu_1)}{\partial (\Sigma_{1,1}^{-1}\mu_1)}\mu_2^{\top}.
\end{eqnarray*}
In addition,
\begin{eqnarray*}
\frac{\partial p(x_1;\Sigma_{1,1}^{-1}\mu_1)}{\partial (\Sigma_{1,1}^{-1}\mu_1)}&=&\frac{\partial }{\partial (\Sigma_{1,1}^{-1}\mu_1)}\frac{ \psi(x_1,\mu_1,{\Sigma}_{1,1}) }{\psi(x_1,\mu_1,{\Sigma}_{1,1})+\psi(x_1,-\mu_1,{\Sigma}_{1,1})}\\
&=&\frac{\partial \psi(x_1,\mu_1,{\Sigma}_{1,1})}{\partial (\Sigma_{1,1}^{-1}\mu_1)}\frac{ 1 }{\psi(x_1,\mu_1,{\Sigma}_{1,1})+\psi(x_1,-\mu_1,{\Sigma}_{1,1})}\\
&&
-\frac{ \psi(x_1,\mu_1,{\Sigma}_{1,1}) }{[\psi(x_1,\mu_1,{\Sigma}_{1,1})+\psi(x_1,-\mu_1,{\Sigma}_{1,1})]^2}\frac{\partial }{\partial (\Sigma_{1,1}^{-1}\mu_1)}(\psi(x_1,\mu_1,{\Sigma}_{1,1})+\psi(x_1,-\mu_1,{\Sigma}_{1,1}))\\
&=&p(x_1;\Sigma_{1,1}^{-1}\mu_1)x_1-p(x_1;\Sigma_{1,1}^{-1}\mu_1)(2p(x_1;\Sigma_{1,1}^{-1}\mu_1)-1)x_1\\
&=&2p(x_1;\Sigma_{1,1}^{-1}\mu_1)(1-p(x_1;\Sigma_{1,1}^{-1}\mu_1))x_1.
\end{eqnarray*}
We know that $\sqrt{n_1+n_3}([\widehat{\Sigma_{1,1}^{-1}\mu_1},\widehat\mu_2]-[(\Sigma_{1,1}^*)^{-1}\mu_1^*,\mu_2^*])$ asymptotically follows $N(\textbf{0},A^{-1}B(A^{-1})^{\top})$ where
\begin{eqnarray*}
A&=&\frac{\partial^2}{\partial [(\Sigma_{1,1}^{-1}\mu_1),\mu_2]^2}\mathbb{E}\| x_2-\phi(x_1) \|_2^2, \\
B&=&\mathbb{E}\left(\frac{\partial}{\partial [(\Sigma_{1,1}^{-1}\mu_1),\mu_2]}\| x_2-\phi(x_1) \|_2^2 \right)\left(\frac{\partial}{\partial [(\Sigma_{1,1}^{-1}\mu_1),\mu_2]}\| x_2-\phi(x_1) \|_2^2 \right)^{\top}.
\end{eqnarray*}
Thus we have
\begin{eqnarray*}
A:=\begin{bmatrix}
A_{1,1} & A_{1,2} \\
A_{2,1} & A_{2,2}
\end{bmatrix}
\end{eqnarray*}
with
\begin{eqnarray*}
A_{1,1}&=&\frac{\partial^2}{\partial (\Sigma_{1,1}^{-1}\mu_1)^2}\mathbb{E}\| x_2-\phi(x_1) \|_2^2=8\|\mu_2\|^2\mathbb{E}\left(\frac{\partial p(x_1;\Sigma_{1,1}^{-1}\mu_1)}{\partial (\Sigma_{1,1}^{-1}\mu_1)}\right)\left(\frac{\partial p(x_1;\Sigma_{1,1}^{-1}\mu_1)}{\partial (\Sigma_{1,1}^{-1}\mu_1)}\right)^{\top}\\
A_{1,2}&=&\frac{\partial^2}{\partial (\Sigma_{1,1}^{-1}\mu_1)\partial\mu_2}\mathbb{E}\| x_2-\phi(x_1) \|_2^2=4\mathbb{E}\left[(2p(x_1)-1)\frac{\partial p(x_1;\Sigma_{1,1}^{-1}\mu_1)}{\partial (\Sigma_{1,1}^{-1}\mu_1)}\right]\mu_2^{\top}\\
A_{2,2}&=&\frac{\partial^2}{\partial \mu_2^2}\mathbb{E}\| x_2-\phi(x_1) \|_2^2=2\mathbb{E}(2p(x_1)-1)^2I_{d_2}
\end{eqnarray*}

Using block matrix inversion on $A$, we have
\begin{eqnarray*}
A^{-1}=\begin{bmatrix}
(A_{1,1}-A_{1,2}A_{2,2}^{-1}A_{2,1})^{-1} & -(A_{1,1}-A_{1,2}A_{2,2}^{-1}A_{2,1})^{-1}A_{1,2}A_{2,2}^{-1}\\
-A_{2,2}^{-1}A_{2,1}(A_{1,1}-A_{1,2}A_{2,2}^{-1}A_{2,1})^{-1} & ...
\end{bmatrix}.
\end{eqnarray*}
As a result, the Bahadur representation of $\widehat{\Sigma_{1,1}^{-1}\mu_1}$ is
\begin{eqnarray}\label{eqn:bahadur}
\widehat{\Sigma_{1,1}^{-1}\mu_1}-(\Sigma_{1,1}^*)^{-1}\mu_1^*&=& \frac{-(A_{1,1}-A_{1,2}A_{2,2}^{-1}A_{2,1})^{-1}}{n_1+n_3}\sum_{S_1,S_3} \frac{\partial}{\partial\Sigma_{1,1}^{-1}\mu_1}\| x_2-\phi(x_1,\mu_1^*,\Sigma_{1,1}^*) \|_2^2\\
&&+\frac{(A_{1,1}-A_{1,2}A_{2,2}^{-1}A_{2,1})^{-1}A_{1,2}A_{2,2}^{-1}}{n_1+n_3}\sum_{S_1,S_3}\frac{\partial}{\partial\mu_2}\| x_2-\phi(x_1,\mu_1^*,\Sigma_{1,1}^*) \|_2^2+o.
\end{eqnarray}

Expanding $A_{1,2}$ in the Bahadur representation and then
\begin{eqnarray*}
\mathbb{E}\|\widehat{\Sigma_{1,1}^{-1}\mu_1}-(\Sigma_{1,1}^*)^{-1}\mu_1^* \|^2=O\left(\frac{d_1}{n_1+n_3}\right).
\end{eqnarray*}

Finally, in terms of the regret, following Lemma 6.3 of \cite{dan2020sharp}, we have
\begin{eqnarray*}
&&\mathbb{E}R(2p(\cdot,\widehat{\mu}_1)-1)-R(2p(\cdot,\mu_1^*)-1)=O\left(\mathbb{E}\|\widehat{\Sigma_{1,1}^{-1}\mu_1}-(\Sigma_{1,1}^*)^{-1}\mu_1^* \|^2\right)=O\left(\frac{d_1}{n_1+n_3}\right).
\end{eqnarray*}

\end{proof}

\subsection{Proof in Section 3.1}
\begin{proof}[Proof of Theorem \ref{thm:mixed:upper}]
Theorem \ref{thm:mixed:upper} is established upon Assumption \ref{assumption} and \ref{assumption:more}.

When CI condition holds, the covariance matrix $A_{1,2}$ in Example \ref{example:mixed_gaussian} is of rank 1. From the pretext task and the family of $\phi$ we choose, following the same arguments as Example \ref{example:mixed_gaussian}\footnote{ The way of doing Taylor expansion is the same for Example \ref{example:mixed_gaussian} and other models, and the Bahadur representation (\ref{eqn:bahadur}) is the same. }, under Assumption \ref{assumption:more},
\begin{eqnarray*}
\frac{1}{\|\beta^*\|^2}\mathbb{E}\|\widehat{\beta}-\beta^*\|^2=O\left(\frac{  d_1}{n_1+n_3}\right).
\end{eqnarray*}
On the other hand, as we mentioned before, the output of $\widehat{\phi}$ is always in the same direction, thus there is no further error involved in the downstream task in terms of the misclassification rate for both plugin estimator and logistic regression.

Denote $ \widehat{\mu}_1 $ and $\widehat{\mu}_2$ are the parameter associated with $\widehat{\phi}$, then since
\begin{eqnarray*}
R^*=R(\phi^*W^*)=R(2p(\cdot,\beta^*)-1),
\end{eqnarray*}
based on Assumption \ref{assumption}, we have
\begin{eqnarray*}
R(\widehat{\phi}\widehat{W})-R^*=O\left(\frac{  d_1}{n_1+n_3}\right).
\end{eqnarray*}
\end{proof}

\begin{proof}[Proof of Theorem \ref{thm:linear}, Square Loss]
Theorem \ref{thm:linear} is established upon Assumption \ref{assumption} and \ref{assumption:more}.

Denote $(X_1,X_2)$ as the data matrix for $S_1, S_3$ (without response), and $(X_1',X_2',Y')$ as the data matrix for $S_1, S_2$. Denote $\Sigma_{i,j}=\mathbb{E}x_ix_j^{\top}$ for $i,j\in\{1,2\}$. Also denote $\Sigma_{i,y}$ as $\mathbb{E}x_iy$ for $i\in\{1,2\}$. 

We first look at the asymptotics of $\widehat{\phi}\widehat{W}$. From the problem setup, we can directly solve $\widehat{\phi}$:
\begin{eqnarray*}
\widehat{\phi} = (X_1^{\top}X_1)^{-1}X_1^{\top}X_2 \rightarrow \Sigma_{1,1}^{-1}\Sigma_{1,2},
\end{eqnarray*}
and further write down $\widehat{W}$:
\begin{eqnarray*}
\widehat{W}=( \widehat{\phi}(X_1')^{\top} \widehat{\phi}(X_1'))^{-1}\widehat{\phi}(X_1')^{\top}Y'\rightarrow (\Sigma_{2,1}\Sigma_{1,1}^{-1}\Sigma_{1,2})^{-1}\Sigma_{2,1}\Sigma_{1,1}^{-1}\Sigma_{1,y}.
\end{eqnarray*}
Thus
\begin{eqnarray*}
\widehat{\theta}=\widehat{\phi}\widehat{W}\rightarrow \Sigma_{1,1}^{-1}\Sigma_{1,2}(\Sigma_{2,1}\Sigma_{1,1}^{-1}\Sigma_{1,2})^{-1}\Sigma_{2,1}\Sigma_{1,1}^{-1}\Sigma_{1,y}.
\end{eqnarray*}

We then study the convergence rate. Denote $\theta_0$ as the $\widehat\theta$ obtained when $n_i\rightarrow\infty$ for all $i=1,2,3$. For the pretext task, one can see that
\begin{eqnarray*}
\widehat{\phi}-{\phi}^*&=&( \widehat{\phi}(X_1')^{\top} \widehat{\phi}(X_1'))^{-1}\widehat{\phi}(X_1')^{\top}Y'-{\phi}^*\\
&=&\Sigma_{1,1}^{-1}\left(\frac{X_1^{\top}X_2}{n_1+n_3}-\Sigma_{1,2}\right)+\Sigma_{1,1}^{-1}\left(\frac{X_1^{\top}X_1}{n_1+n_3}-\Sigma_{1,1}\right)\Sigma_{1,1}^{-1}\Sigma_{1,2}+o.
\end{eqnarray*}
\begin{eqnarray*}
\widehat{W}-W^*&=&( \widehat{\phi}(X_1')^{\top} \widehat{\phi}(X_1'))^{-1}\widehat{\phi}(X_1')^{\top}Y'-W^*\\
&=&\left( \frac{\phi^*(X_1')^{\top}\phi^*(X_1')}{n_1+n_2} + (\widehat{\phi}-\phi^*)^{\top}\Sigma_{1,2}+\Sigma_{2,1}(\widehat{\phi}-\phi^*)+o  \right)^{-1}\\&&\cdot\left( \frac{\phi^*(X_1')^{\top}Y'}{n_1+n_2}+ (\widehat{\phi}-\phi^*)^{\top}\Sigma_{1,1}\theta_0 +o\right),
\end{eqnarray*}
where
\begin{eqnarray*}
&&\left( \frac{\phi^*(X_1')^{\top}\phi^*(X_1')}{n_1+n_2} + (\widehat{\phi}-\phi^*)^{\top}\Sigma_{1,2}+\Sigma_{2,1}(\widehat{\phi}-\phi^*)+o  \right)^{-1}-\left(\frac{\phi^*(X_1')^{\top}\phi^*(X_1')}{n_1+n_2}\right)^{-1}\\
&=&\left(\frac{\phi^*(X_1')^{\top}\phi^*(X_1')}{n_1+n_2}\right)^{-1}\left( (\widehat{\phi}-\phi^*)^{\top}\Sigma_{1,2}+\Sigma_{2,1}(\widehat{\phi}-\phi^*) \right)\left(\frac{\phi^*(X_1')^{\top}\phi^*(X_1')}{n_1+n_2}\right)^{-1}+o.
\end{eqnarray*}
As a result, denoting $\widehat{\theta}=\widehat{\phi}^{\top}\widehat{W}$, we have
\begin{eqnarray*}
\widehat{\theta}-\theta_0 &=& (\widehat{\phi}-\phi^*)^{\top}W^*+(\phi^*)^{\top}(\widehat{W}-W^*)+o,
\end{eqnarray*}
therefore, 
\begin{eqnarray*}
\mathbb{E}\|\widehat{\theta}-\theta_0\|^2=O\left( \frac{d_2}{n_1+n_2} + \frac{d_1+d_2}{n_1+n_3} \right).
\end{eqnarray*}
\end{proof}
\begin{proof}[Proof of Theorem \ref{thm:linear}, Logistic Regression]
Theorem \ref{thm:linear} is established upon Assumption \ref{assumption} and \ref{assumption:more}.

The derivation for the pretext task is the same as the one in square loss. In the downstream task, denote $\xi_1(x_1,\phi, W)=\phi(x_1)\frac{1}{1+e^{\phi(x_1)^{\top}W}}$ and $\xi_2(x_1,\phi, W)=\phi(x_1)\frac{1}{1+e^{-\phi(x_1)^{\top}W}}$, then
\begin{eqnarray*}
&&\frac{1}{n_1+n_2}\sum_{S_1,S_2}\left[ p(x_1)\xi_1(x_1,\phi^* ,W^*) -(1-p(x_1))\xi_2(x_1,\phi^*, W^*)\right]\\
&=&\frac{1}{n_1+n_2}\sum_{S_1,S_2}\left[ p(x_1)\xi_1(x_1,\phi^*, W^*) -(1-p(x_1))\xi_2(x_1,\phi^*, W^*)\right]\\
&&-\frac{1}{n_1+n_2}\sum_{S_1,S_2}\left[1\{y=1\}\xi_1(x_1,\widehat{\phi} ,\widehat{W}) -1\{y=-1\}\xi_2(x_1,\widehat{\phi},\widehat{W})\right]\\
&=&\frac{1}{n_1+n_2}\sum_{S_1,S_2}\left[ p(x_1)\xi_1(x_1,\phi^*, W^*) -(1-p(x_1))\xi_2(x_1,\phi^*, W^*)\right]\\
&&-\frac{1}{n_1+n_2}\sum_{S_1,S_2}\left[p(x_1)\xi_1(x_1,\phi^* ,\widehat{W}) -(1-p(x_1))\xi_2(x_1,\phi^*, \widehat{W})\right]\\
&&+\frac{1}{n_1+n_2}\sum_{S_1,S_2}\left[p(x_1)\xi_1(x_1,\phi^*,\widehat{W}) -(1-p(x_1))\xi_2(x_1,\phi^*, \widehat{W})\right]\\
&&-\frac{1}{n_1+n_2}\sum_{S_1,S_2}\left[p(x_1)\xi_1(x_1,\widehat{\phi} ,\widehat{W}) -(1-p(x_1))\xi_2(x_1,\widehat{\phi}, \widehat{W})\right]\\
&&+\frac{1}{n_1+n_2}\sum_{S_1,S_2}\left[p(x_1)\xi_1(x_1,\widehat{\phi}, \widehat{W}) -(1-p(x_1))\xi_2(x_1,\widehat{\phi} ,\widehat{W})\right]\\
&&-\frac{1}{n_1+n_2}\sum_{S_1,S_2}\left[1\{y=1\}\xi_1(x_1,\widehat{\phi}, \widehat{W}) -1\{y=-1\}\xi_2(x_1,\widehat{\phi}, \widehat{W})\right].
\end{eqnarray*}
Observe that with probability tending to 1,
\begin{eqnarray*}
&&\frac{1}{n_1+n_2}\sum_{S_1,S_2}\left[ p(x_1)\xi_1(x_1,\phi^* ,W^*) -(1-p(x_1))\xi_2(x_1,\phi^*, W^*)\right]\\
&&-\frac{1}{n_1+n_2}\sum_{S_1,S_2}\left[p(x_1)\xi_1(x_1,\phi^* ,\widehat{W}) -(1-p(x_1))\xi_2(x_1,\phi^*, \widehat{W})\right]\\
&=&\underbrace{\mathbb{E}\left[ p(x_1)\frac{\partial \xi_1}{\partial W}(x_1,\phi^*,W^*)-(1-p(x_1))\frac{\partial \xi_2}{\partial W}(x_1,\phi^*,W^*) \right]}_{:=A}(W^*-\widehat{W})+o,
\end{eqnarray*}
and
\begin{eqnarray*}
&&\frac{1}{n_1+n_2}\sum_{S_1,S_2}\left[p(x_1)\xi_1(x_1,\widehat{\phi} ,\widehat{W}) -(1-p(x_1))\xi_2(x_1,\widehat{\phi}, \widehat{W})\right]\\
&&-\frac{1}{n_1+n_2}\sum_{S_1,S_2}\left[1\{y=1\}\xi_1(x_1,\widehat{\phi}, \widehat{W}) -1\{y=-1\}\xi_2(x_1,\widehat{\phi} ,\widehat{W})\right]
\end{eqnarray*}
is a random noise with variance $O(d_2/(n_1+n_2))$. In addition,
\begin{eqnarray*}
&&\frac{1}{n_1+n_2}\sum_{S_1,S_2}\left[p(x_1)\xi_1(x_1,\phi^*,\widehat{W}) -(1-p(x_1))\xi_2(x_1,\phi^*, \widehat{W})\right]\\
&&-\frac{1}{n_1+n_2}\sum_{S_1,S_2}\left[p(x_1)\xi_1(x_1,\widehat{\phi} ,\widehat{W}) -(1-p(x_1))\xi_2(x_1,\widehat{\phi}, \widehat{W})\right]\\
&=&\frac{1}{n_1+n_2}\sum_{S_1,S_2}\left[p(x_1)\xi_1(x_1,\phi^*,W^*) -(1-p(x_1))\xi_2(x_1,\phi^*, W^*)\right]\\
&&-\frac{1}{n_1+n_2}\sum_{S_1,S_2}\left[p(x_1)\xi_1(x_1,\widehat{\phi} ,W^*) -(1-p(x_1))\xi_2(x_1,\widehat{\phi}, W^*)\right]+o\\
&=&\underbrace{\mathbb{E}\left[ p(x_1)\frac{\partial \xi_1}{\partial \phi W}(x_1,\phi^*,W^*)+(1-p(x_1))\frac{\partial \xi_2}{\partial \phi W}(x_1,\phi^*,W^*) \right]}_{:=B}(\widehat{\phi}-\phi^*)W^*+o
\end{eqnarray*}
Thus we have 
\begin{eqnarray*}
\widehat{W}-W^*&=& -\frac{A^{-1}}{n_1+n_2}\sum_{S_1,S_2}\left[ p(x_1)\xi_1(x_1,\phi^* ,W^*) -(1-p(x_1))\xi_2(x_1,\phi^*, W^*)\right]\\
&&+A^{-1}B(\widehat{\phi}-\phi^*)W^*\\
&&+\frac{A^{-1}}{n_1+n_2}\sum_{S_1,S_2}\left[p(x_1)\xi_1(x_1,\widehat{\phi}, \widehat{W}) -(1-p(x_1))\xi_2(x_1,\widehat{\phi}, \widehat{W})\right]\\
&&-\frac{A^{-1}}{n_1+n_2}\sum_{S_1,S_2}\left[1\{y=1\}\xi_1(x_1,\widehat{\phi}, \widehat{W}) -1\{y=-1\}\xi_2(x_1,\widehat{\phi} ,\widehat{W})\right]+o.
\end{eqnarray*}
As a result, taking $\widehat{\theta}=\widehat{\phi}\widehat{W}$ and $\theta_0=\phi^*W^*$,
\begin{eqnarray*}
\mathbb{E}\|\widehat{\theta}-\theta_0\|^2=O\left(\frac{d_2}{n_1+n_2}+\frac{d_1+d_2}{n_1+n_3}\right).
\end{eqnarray*}
\end{proof}
\subsection{Proof in Section 3.2}

\begin{proof}[Proof of Theorem \ref{thm:mix:adv}, Logistic Regression, Upper bound]
Theorem \ref{thm:mix:adv} is built upon Assumption \ref{assumption}, \ref{assumption:more}, and \ref{assumption:extra}. Assumption \ref{assumption}, \ref{assumption:more} ensures the performance of the SSL in clean training, and Assumption \ref{assumption:extra} regulates the adversarial training.

Below is a summary of the proof:
\begin{itemize}
    \item Part 1: we show that $\widehat{\theta}$ is consistent.
    \item Part 2: given the consistency results from Part 1, we can present the Bahadur representation of $\widehat{\theta}$.
    \item Part 3: we figure out $\widehat{p}$ from clean training, and take it into the Bahadur representation to get the final convergence result.
\end{itemize}

We first use the data model in Example \ref{example:mixed_gaussian} to go through the proof, then discuss on how to generalize it. The extra moment conditions mentioned in the theorem statement are mentioned when we generalize the proof.

\textbf{Part 1:} Our first aim is to show that $\widehat\theta$ is consistent, i.e., $\widehat\theta\rightarrow\theta^*$. To achieve this, Denoting $\xi_1(x_1,\theta)=\left(x_1-\epsilon \frac{\theta}{\|\theta\|}\right) \frac{1}{1+e^{x_1^{\top}\theta-\epsilon\|\theta\|}}$ and $\xi_2(x_1,\theta)=\left(x_1+\epsilon \frac{\theta}{\|\theta\|}\right) \frac{1}{1+e^{-x_1^{\top}\theta-\epsilon\|\theta\|}}$, since the adversarial logistic loss is strongly convex, there exists some constant $C>0$ such that
\begin{eqnarray*}
&&\inf_{\|\theta-\theta^*\|\geq\varepsilon}\bigg\|\mathbb{E}\left[ p(x_1)\xi_1(x_1,\theta^*) -(1-p(x_1))\xi_2(x_1,\theta^*)\right]-\mathbb{E}\left[ p(x_1)\xi_1(x_1,\theta) -(1-p(x_1))\xi_2(x_1,\theta) \right]\bigg\|\\
&=&\inf_{\|\theta-\theta^*\|\geq\varepsilon}\bigg\|\mathbb{E}\left[ p(x_1)\xi_1(x_1,\theta) -(1-p(x_1))\xi_2(x_1,\theta) \right]\bigg\|\\
&\geq& C\varepsilon^2.
\end{eqnarray*}
Furthermore, with probability tending to 1, we have
\begin{eqnarray*}
\bigg\|\mathbb{E}\left[ p(x_1)\xi_1(x_1,\widehat{\theta})  -(1-p(x_1))\xi_2(x_1,\widehat\theta)  \right] - \frac{1}{\sum n_i}\sum_{S_1,S_2,S_3,S_4}\left[ p(x_1)\xi_1(x_1,\widehat\theta)  -(1-p(x_1))\xi_2(x_1,\widehat\theta)  \right]\bigg\|
\end{eqnarray*}
is smaller than some $\tau\rightarrow 0$ in $\sum{n_i}$. We further compare
\begin{eqnarray*}
\frac{1}{\sum n_i}\sum_{S_1,S_2,S_3,S_4}\left[ p(x_1)\xi_1(x_1,\widehat\theta)  -(1-p(x_1))\xi_2(x_1,\widehat\theta)  \right]
\end{eqnarray*}
to the first-order optimality condition, i.e.
\begin{eqnarray*}
\frac{1}{\sum n_i}\sum_{S_1,S_2,S_3,S_4}\left[1_{\{y=1\}}\xi_1(x_1,\widehat\theta)  -1_{\{y=-1\}}\xi_2(x_1,\widehat\theta)  \right]=\textbf{0}.
\end{eqnarray*}

Since $S_1$ and $S_2$ contains labels, we have with probability tending to 1, for some constant $r>0$,
\begin{eqnarray*}
&&\sup_{\|\theta-\theta^*\|\leq r}\bigg\|\frac{1}{\sum n_i}\sum_{S_1,S_2}\left[ p(x_1)\xi_1(x_1,\theta)  -(1-p(x_1))\xi_2(x_1,\theta) \right]\\
&&\qquad\qquad\qquad-\frac{1}{\sum n_i}\sum_{S_1,S_2}\left[1_{\{y=1\}}\xi_1(x_1,\theta)  -1_{\{y=-1\}}\xi_2(x_1,\theta)  \right]\bigg\|\rightarrow 0.
\end{eqnarray*}
In terms of $S_3$ and $S_4$, the labels are imputed from $\widehat{p}$, thus
\begin{eqnarray*}
&&\sup_{\|\theta-\theta^*\|\leq r}\bigg\|\frac{1}{\sum n_i}\sum_{S_3,S_4}\left[ p(x_1)\xi_1(x_1,\theta)  -(1-p(x_1))\xi_2(x_1,\theta) \right]\\
&&\qquad\qquad\qquad-\frac{1}{\sum n_i}\sum_{S_3,S_4}\left[\widehat p(x_1)\xi_1(x_1,\theta)  -(1-\widehat p(x_1))\xi_2(x_1,\theta)  \right]\bigg\|\\
&=&\sup_{\|\theta-\theta^*\|\leq r}\bigg\|\frac{1}{\sum n_i}\sum_{S_3,S_4}\left[ (\widehat{p}(x_1)-p(x_1))\left(\xi_1(x_1,\theta) +\xi_2(x_1,\theta) \right) \right]\bigg\|,
\end{eqnarray*}
which also converges to zero since $\widehat{p}\rightarrow p$ and each dimension of $x_1$ has finite fourth moment. Further following similar argument as for $S_1$ and $S_2$, we have
\begin{eqnarray*}
&&\sup_{\|\theta-\theta^*\|\leq R}\bigg\|\frac{1}{\sum n_i}\sum_{S_3,S_4}\left[ p(x_1)\xi_1(x_1,\theta)  -(1-p(x_1))\xi_2(x_1,\theta) \right]\\
&&\qquad\qquad\qquad-\frac{1}{\sum n_i}\sum_{S_3,S_4}\left[1_{\{y=1\}}\xi_1(x_1,\theta)  -1_{\{y=-1\}}\xi_2(x_1,\theta)  \right]\bigg\|\rightarrow 0.
\end{eqnarray*}
Therefore, combining all the above results, we have
\begin{eqnarray*}
Pr( \|\widehat{\theta}-\theta^*\|\geq \varepsilon ) &\leq& P\left( \bigg\|\mathbb{E}\left[ p(x_1)\xi_1(x_1,\widehat\theta) -(1-p(x_1))\xi_2(x_1,\widehat\theta) \right]\bigg\| \geq C\varepsilon^2 \right),
\end{eqnarray*}
and $\bigg\|\mathbb{E}\left[ p(x_1)\xi_1(x_1,\widehat\theta) -(1-p(x_1))\xi_2(x_1,\widehat\theta) \right]\bigg\|\rightarrow 0$ in probability, thus with probability tending to 1, $\|\widehat{\theta}-\theta^*\|\rightarrow 0.$

\textbf{Part 2:} Given the consistency result, we further consider the convergence rate as a function of $\widehat{p}-p$. We have
\begin{eqnarray*}
&&\sum_{S_1,S_2,S_3,S_4}\left[ p(x_1)\xi_1(x_1,\theta^*) -(1-p(x_1))\xi_2(x_1,\theta^*)\right]\\
&=&\sum_{S_1,S_2,S_3,S_4}\left[ p(x_1)\xi_1(x_1,\theta^*) -(1-p(x_1))\xi_2(x_1,\theta^*) \right]-\underbrace{\sum_{S_1,S_2,S_3,S_4}\left[1_{\{y=1\}}\xi_1(x_1,\widehat\theta) -1_{\{y=-1\}}\xi_2(x_1,\widehat\theta)\right]}_{=\textbf{0}\text{ by optimality condition}}\\
&=&\sum_{S_1,S_2}\left[ p(x_1)\xi_1(x_1,\widehat\theta) -(1-p(x_1))\xi_2(x_1,\widehat\theta) \right]-\sum_{S_1,S_2}\left[1_{\{y=1\}}\xi_1(x_1,\widehat\theta) -1_{\{y=-1\}}\xi_2(x_1,\widehat\theta) \right]\\
&&+\sum_{S_1,S_2,S_3,S_4}\left[ p(x_1)\xi_1(x_1,\theta^*) -(1-p(x_1))\xi_2(x_1,\theta^*) \right]-\sum_{S_1,S_2,S_3,S_4}\left[{p}(x_1)\xi_1(x_1,\widehat\theta) -(1- {p}(x_1))\xi_2(x_1,\widehat\theta) \right]\\
&&+\sum_{S_3,S_4}\left[{p}(x_1)\xi_1(x_1,\widehat\theta)  -(1- {p}(x_1))\xi_2(x_1,\widehat\theta)  \right]-\sum_{S_3,S_4}\left[ \widehat{p}(x_1)\xi_1(x_1,\widehat\theta) -(1- \widehat{p}(x_1))\xi_2(x_1,\widehat\theta)  \right]\\
&&+\sum_{S_3,S_4}\left[ \widehat{p}(x_1)\xi_1(x_1,\widehat\theta) -(1- \widehat{p}(x_1))\xi_2(x_1,\widehat\theta)  \right]-\sum_{S_3,S_4}\left[1_{\{y=1\}}\xi_1(x_1,\widehat\theta)  -1_{\{y=-1\}}\xi_2(x_1,\widehat\theta) \right]\\
&:=&A_1+A_2+A_3+A_4=A_0
\end{eqnarray*}

$A_1$ is a random variable with noise variance $O(n_1+n_2)$. $A_2$ measures the difference between $\theta$ and $\theta^*$. $A_3$ measures the difference between $\widehat{p}$ and $p$. $A_4$ is a random variable with noise variance $O(n_3+n_4)$. With probability tending to 1,


\begin{eqnarray*}
A_2&=&-\sum_{S_1,S_2,S_3,S_4}\left[ p(x_1)\xi_1'(x_1,\theta^*)(\theta-\theta^*)-(1-p(x_1))\xi_2'(x_1,\theta^*)(\theta-\theta^*) \right]+o\\
&=& -\left(\sum n_i\right)\mathbb{E}\left[ p(x_1)\xi_1'(x_1,\theta^*)- (1-p(x_1))\xi_2'(x_1,\theta^*)\right](\theta-\theta^*)+O(\sqrt{\sum n_i}\|\theta-\theta^*\| )+o.
\end{eqnarray*}

Therefore we have
\begin{eqnarray*}
\widehat\theta-\theta^*
&=&\frac{1}{\sum n_i} \underbrace{\left( \mathbb{E}\left[ p(x_1)\xi_1'(\theta^*)- (1-p(x_1))\xi_2'(\theta^*)\right]\right)^{-1}}_{:=\Gamma^{-1}}\left(A_1+A_3+A_4-A_0\right)\\
&=&\underbrace{\frac{\Gamma^{-1}}{\sum n_i}\sum_{S_1,S_2}\left(1_{\{y=1\}}-p(x_1)\right)\xi_1(x_1,\theta^*)-(1_{ \{y=-1\}}-1+p(x_1))\xi_2(x_1,\theta^*)}_{=O_p\left( \sqrt{ d_1(n_1+n_2)}/\sum n_i  \right)}\\
&&+\underbrace{\frac{\Gamma^{-1}}{\sum n_i}\sum_{S_3,S_4}(p(x_1)-\widehat{p}(x_1))\left( \xi_1(x_1,\theta^*)+\xi_2(x_1,\theta^*) \right)}_{:=\Delta}\\
&&+\underbrace{\frac{\Gamma^{-1}}{\sum n_i}\sum_{S_3,S_4}\left(1_{\{y=1\}}-\widehat p(x_1)\right)\xi_1(x_1,\theta^*)-(1_{ \{y=-1\}}-1+\widehat p(x_1))\xi_2(x_1,\theta^*)}_{=O_p\left( \sqrt{ d_1(n_3+n_4)}/\sum n_i  \right) }\\
&&-\underbrace{\frac{\Gamma^{-1}}{\sum n_i}\sum_{S_1,S_2,S_3,S_4}\left[ p(x_1)\xi_1(x_1,\theta^*)-(1-p(x_1))\xi_2(x_1,\theta^*) \right]}_{=O_p\left(\sqrt{d_1/\sum n_i}\right)}+o.
\end{eqnarray*}

 \textbf{Part 3:} We further using the construction of $\widehat{p}$ to bound $\Delta$. 
    As mentioned in Example \ref{example:mixed:label}, we use $\widehat\phi$ to obtain $\mu_1$. 
    We know that $\|\widehat{\mu}_1-\mu_1\|/\|\mu_1\|=O_p( d_1/(n_1+n_3) )$ and $\widehat{\mu}_1-\mu_1$ can be represented using Bahadur representation as well. For $(x_1)$ in $S_4$, it is independent to $\widehat{\Sigma_{1,1}^{-1}\mu_1}$, thus
    \begin{eqnarray*}
    \mathbb{E} \left\|\Gamma^{-1}\left\langle \widehat{\Sigma_{1,1}^{-1}\mu_1} -(\Sigma_{1,1}^*)^{-1}\mu_1, \frac{\partial p(x_1)}{\partial {\Sigma_{1,1}^{-1}\mu_1}} \right\rangle\left( \xi_1(x_1,\theta^*)+\xi_2(x_1,\theta^*) \right)\right\|^2=O\left(\frac{d_1^2}{n_1+n_3}\right).
    \end{eqnarray*}
    For two samples $(x_1)$ and $(x_1')$ in $S_4$, we have
    \begin{eqnarray*}
    &&\mathbb{E}\bigg[\left\langle \widehat{\Sigma_{1,1}^{-1}\mu_1} -(\Sigma_{1,1}^*)^{-1}\mu_1, \frac{\partial p(x_1)}{\partial {\Sigma_{1,1}^{-1}\mu_1}} \right\rangle\left\langle \widehat{\Sigma_{1,1}^{-1}\mu_1} -(\Sigma_{1,1}^*)^{-1}\mu_1, \frac{\partial p(x_1')}{\partial {\Sigma_{1,1}^{-1}\mu_1}} \right\rangle\\
    &&\qquad\times\left( \xi_1(x_1,\theta^*)+\xi_2(x_1,\theta^*) \right)^{\top}\Gamma^{-2}\left( \xi_1(x_1',\theta^*)+\xi_2(x_1',\theta^*) \right)\bigg]\\
    &=&O\left(\frac{d_1}{n_1+n_3}\right).
    \end{eqnarray*}
    Thus
    \begin{eqnarray*}
    &&\mathbb{E}\left\|\frac{\Gamma^{-1}}{\sum n_i}\sum_{S_4}\left\langle \widehat{\Sigma_{1,1}^{-1}\mu_1} -(\Sigma_{1,1}^*)^{-1}\mu_1, \frac{\partial p(x_1)}{\partial {\Sigma_{1,1}^{-1}\mu_1}} \right\rangle\left( \xi_1(x_1,\theta^*)+\xi_2(x_1,\theta^*) \right)\right\|^2\\
    &=&O\left(\frac{d_1}{n_1+n_3}\left(\frac{n_4}{\sum n_i}\right)^2 + \frac{d_1^2}{n_1+n_3}\frac{n_4}{(\sum n_i)^2} \right)\\
    &=&O\left(\frac{d_1}{n_1+n_3}\right).
    \end{eqnarray*}  
    For $(x_1,x_2)$ in $S_3$, it is correlated to $\widehat{\Sigma_{1,1}^{-1}\mu_1}$, thus
    \begin{eqnarray*}
    \mathbb{E} \left\|\Gamma^{-1}\left\langle \widehat{\Sigma_{1,1}^{-1}\mu_1} -(\Sigma_{1,1}^*)^{-1}\mu_1, \frac{\partial p(x_1)}{\partial {\Sigma_{1,1}^{-1}\mu_1}} \right\rangle\left( \xi_1(x_1,\theta^*)+\xi_2(x_1,\theta^*) \right)\right\|^2=O\left(\frac{d_1^2}{n_1+n_3}\right)+O\left(\frac{d_1^3}{(n_1+n_3)^2}\right),
    \end{eqnarray*}
    and for two samples $(x_1,x_2)$ and $(x_1',x_2')$ in $S_3$, it becomes
    \begin{eqnarray*}
    &&\mathbb{E}\bigg[\left\langle \widehat{\Sigma_{1,1}^{-1}\mu_1} -(\Sigma_{1,1}^*)^{-1}\mu_1, \frac{\partial p(x_1)}{\partial {\Sigma_{1,1}^{-1}\mu_1}} \right\rangle\left\langle \widehat{\Sigma_{1,1}^{-1}\mu_1} -(\Sigma_{1,1}^*)^{-1}\mu_1, \frac{\partial p(x_1')}{\partial {\Sigma_{1,1}^{-1}\mu_1}} \right\rangle\\
    &&\qquad\times\left( \xi_1(x_1,\theta^*)+\xi_2(x_1,\theta^*) \right)^{\top}\Gamma^{-2}\left( \xi_1(x_1',\theta^*)+\xi_2(x_1',\theta^*) \right)\bigg]\\
    &=&O\left(\frac{d_1}{n_1+n_3}+\frac{d_1^2}{(n_1+n_3)^2}\right).
    \end{eqnarray*}
    Thus although $S_3$ is related to $\widehat\mu_1$, we still have
    \begin{eqnarray*}
    &&\mathbb{E}\left\|\frac{\Gamma^{-1}}{\sum n_i}\sum_{S_3,S_4}\left\langle \widehat{\Sigma_{1,1}^{-1}\mu_1} -(\Sigma_{1,1}^*)^{-1}\mu_1, \frac{\partial p(x_1)}{\partial {\Sigma_{1,1}^{-1}\mu_1}} \right\rangle\left( \xi_1(x_1,\theta^*)+\xi_2(x_1,\theta^*) \right)\right\|^2
    =O\left(\frac{d_1}{n_1+n_3}\right).
    \end{eqnarray*}  
    Combining Part 2 and Part 3 we have
    \begin{eqnarray*}
    \mathbb{E}\|\widehat{\theta}-\theta^* \|^2=O\left(\frac{d_1}{n_1+n_3}\right),
    \end{eqnarray*}
    and further we obtain
    \begin{eqnarray*}
    \mathbb{E}R(\widehat{\theta},\epsilon)-R(\theta^*,\epsilon)=O\left(\frac{d_1}{n_1+n_3}\right).
    \end{eqnarray*}
    
    Assumption \ref{assumption:extra} guarantees that the above analysis can be generalized to other distributions. Furthermore, although the Bahadur representation of $\widehat{\Sigma_{1,1}^{-1}\mu_1}$ involves $X_2$, since $X_1$ and $X_2$ are conditionally independent given $Y=y$, there is no extra requirement on $X_2$.
\end{proof}

\begin{proof}[Proof of Theorem \ref{thm:mix:adv}, Square Loss, Upper bound]
Theorem \ref{thm:mix:adv} is built upon Assumption \ref{assumption}, \ref{assumption:more}, and \ref{assumption:extra}. Assumption \ref{assumption}, \ref{assumption:more} ensures the performance of the SSL in clean training, and Assumption \ref{assumption:extra} regulates the adversarial training.

The proof if similar to the one for logistic regression below and replace $\xi_1$ to  $\xi(x_1,y,\theta)=\left(x_1+\epsilon \frac{\theta}{\|\theta\|}\sgn(y-x_1^{\top}\theta)\right) (y-x_1^{\top}\theta-\sgn(y-x_1^{\top}\theta)\epsilon\|\theta\|)$. The adversarial square loss is strongly convex. 

Assumption \ref{assumption:extra} ensures that the above analysis can generalize to other distributions.
\end{proof}

\begin{proof}[Proof of Theorem \ref{thm:mix:adv}, Linear without CI, Upper bound] Theorem \ref{thm:mix:adv} is built upon Assumption \ref{assumption}, \ref{assumption:more}, and \ref{assumption:extra}. Assumption \ref{assumption}, \ref{assumption:more} ensures the performance of the SSL in clean training, and Assumption \ref{assumption:extra} regulates the adversarial training.

In the proof when CI holds, the Bahadur representation of $\widehat\theta-\theta^*$ does not directly utilize the CI condition. Instead, we use the convergence result of $\widehat p$. Therefore, similarly, we use the convergence result of SSL from Theorem \ref{thm:linear} to obtain the convergence results of $\widehat p$ to apply to Part 3.
\end{proof}

\begin{proof}[Proof of Proposition \ref{prop:mix:adv}]
 \textbf{Logistic Regression} Since $\widehat{p}$ is consistent, one can follow Part 1 of the proof of Theorem \ref{thm:mix:adv} to obtain the consistency result. In terms of the convergence rate, following Part 2 of the proof of Theorem \ref{thm:mix:adv}, we have
\begin{eqnarray*}
\theta-\theta^*=\frac{\Gamma^{-1}}{\sum n_i}\sum_{S_3,S_4}(p(x_1)-\widehat{p}(x_1))\left( \xi_1(x_1,\theta^*)+\xi_2(x_1,\theta^*) \right)+\Delta
\end{eqnarray*}
for some $\mathbb{E}\|\Delta\|^2=O( d_1/(\sum n_i) )$. Since $\mathbb{E}\|\xi_1\|^2=O(d_1)$ and $\mathbb{E}\|\xi_2\|^2=O(d_1)$, we have $\mathbb{E}\|\widehat{\theta}-\theta^*\|^2=O(d_1/(\sum n_i)+\mathbb{E}\|X_1\|^2\|\widehat{p}(X_1)-p(X_1)\|^2)$.

 \textbf{Square Loss} When  $\mathbb{E}\|X_1\|^2\|\widehat{p}(X_1)-p(X_1)\|^2(X_1\theta^*)^2\rightarrow 0$, the convergence rate of $\widehat\theta-\theta^*$ is $\mathbb{E}\|\widehat{\theta}-\theta^*\|^2=O(d_1/(\sum n_i)+\mathbb{E}\|X_1\|^2\|\widehat{p}(X_1)-p(X_1)\|^2(X_1\theta^*)^2)$.
\end{proof}

\subsection{Discussion about Logistic Regression}

Our first goal is to investigate in what is the $\theta^*$ in logistic regression. Assume there are infinite labeled data, the first-order optimality condition is
\begin{eqnarray*}
\mathbb{E}1_{\{Y=1\}}(x+\epsilon\frac{\theta}{\|\theta\|})\frac{1}{1+e^{x^{\top}\theta-\epsilon\|\theta\|}}-\mathbb{E}1_{\{Y=-1\}}(x+\epsilon\frac{\theta}{\|\theta\|})\frac{1}{1+e^{-x^{\top}\theta+\epsilon\|\theta\|}}+\lambda\theta=\textbf{0}
\end{eqnarray*}
From the distribution of $(X,Y)$, we have
\begin{eqnarray*}
\mathbb{E}1_{\{Y=1\}}(x+\epsilon\frac{\theta}{\|\theta\|})\frac{1}{1+e^{x^{\top}\theta-\epsilon\|\theta\|}}+\mathbb{E}1_{\{Y=1\}}(x-\epsilon\frac{\theta}{\|\theta\|})\frac{1}{1+e^{x^{\top}\theta+\epsilon\|\theta\|}}+\lambda\theta=\textbf{0}
\end{eqnarray*}
\begin{eqnarray*}
&&\mathbb{E}1_{\{Y=1\}}x\left( \frac{1}{1+e^{x^{\top}\theta-\epsilon\|\theta\|}}+\frac{1}{1+e^{x^{\top}\theta+\epsilon\|\theta\|}} \right)\\
&=&\mathbb{E}\left\{\mathbb{E}\left[1_{\{Y=1\}}x\left( \frac{1}{1+e^{x^{\top}\theta-\epsilon\|\theta\|}}+\frac{1}{1+e^{x^{\top}\theta+\epsilon\|\theta\|}} \right)\bigg|\theta^{\top}x=u\right]\right\}\\
&=&\mathbb{E}\left\{\mathbb{E}\left[1_{\{Y=1\}}x\left( \frac{1}{1+e^{u-\epsilon R}}+\frac{1}{1+e^{u+\epsilon R}} \right)\bigg|\theta^{\top}x=u\right]\right\}.
\end{eqnarray*}
Since $(X,\theta^{\top}X)$ follows Gaussian with mean $(\mu,\theta^{\top}\mu)$ and variance $\begin{bmatrix}
\Sigma & \Sigma\theta \\
\theta^{\top}\Sigma & \theta^{\top}\Sigma\theta
\end{bmatrix}$, we have
\begin{eqnarray*}
\mathbb{E}(x|\theta^{\top}x=u)=\mu+\frac{\Sigma\theta}{\theta^{\top}\Sigma\theta}(u-\theta^{\top}\mu)=\mu+\frac{\Sigma\theta}{\theta^{\top}\Sigma\theta}u-\frac{\Sigma\theta\theta^{\top}\mu}{\theta^{\top}\Sigma\theta},
\end{eqnarray*}
thus 
denote $\zeta=\zeta(\theta,\mu,\Sigma,\epsilon)=\mathbb{E}\left[ \left( \frac{1}{1+e^{u-\epsilon R}}-\frac{1}{1+e^{u+\epsilon R}} \right)\right]$, we have
\begin{eqnarray*}
&&\mathbb{E}\left\{\mathbb{E}\left[1_{\{Y=1\}}x\left( \frac{1}{1+e^{u-\epsilon R}}+\frac{1}{1+e^{u+\epsilon R}} \right)\bigg|\theta^{\top}x=u\right]\right\}\\
&=&\frac{\Sigma\theta}{\theta^{\top}\Sigma\theta}\mathbb{E}\left[ u\left( \frac{1}{1+e^{u-\epsilon R}}+\frac{1}{1+e^{u+\epsilon R}} \right)\right]+\mu\zeta-\frac{\Sigma\theta\theta^{\top}\mu}{\theta^{\top}\Sigma\theta}\zeta
\end{eqnarray*}

Return to the optimal condition, we have
\begin{eqnarray*}
\mu\zeta-\frac{\Sigma\theta\theta^{\top}\mu}{\theta^{\top}\Sigma\theta}\zeta+\frac{\Sigma\theta}{\theta^{\top}\Sigma\theta}\mathbb{E}\left[ u\left( \frac{1}{1+e^{u-\epsilon R}}+\frac{1}{1+e^{u+\epsilon R}} \right)\right]+\epsilon\frac{\theta}{\|\theta\|}\zeta+\lambda\theta = \textbf{0}
\end{eqnarray*}
Dividing $\zeta$,
\begin{eqnarray*}
\mu-\frac{\theta^{\top}\mu\Sigma\theta}{\theta^{\top}\Sigma\theta}+\frac{\Sigma\theta}{\theta^{\top}\Sigma\theta}\mathbb{E}\left[ u\left( \frac{1}{1+e^{u-\epsilon R}}+\frac{1}{1+e^{u+\epsilon R}} \right)\right]\frac{1}{\zeta}+\epsilon\frac{\theta}{\|\theta\|}+\lambda\theta = \textbf{0}
\end{eqnarray*}

Thus for some constants $A$ and $B$,
\begin{eqnarray*}
\theta=-\left( A\Sigma+B I \right)^{-1}\mu.
\end{eqnarray*}

The above result reveals that, the convergence rate of logistic regression is the same as plugin estimator. However, the relationship between $(A,B)$ in the above formula may be different from the one in plugin estimator, leading to potential bias in adversarial setup.

\section{Proof for Section \ref{sec:discussion} and \ref{sec:appendix:regression}}\label{sec:appendix:proof3}

\subsection{Proof for Section \ref{sec:mixed:nn}}

To provide detailed conditions on the neural network and configurations, we first define some quantities. For two unit vectors $s,t\in\mathbb{R}^{d_1}$, define a function $h$ as 
\begin{eqnarray}
h(s,t)=\mathbb{E}_{w\sim N(0,I_{d_1})} (s^{\top}t1\{w^{\top}s\geq 0, w^{\top}y\geq 0 \})=\frac{s^{\top}t(\pi-\arccos (s^{\top}t))}{2\pi}.
\end{eqnarray}
There are total $n_1+n_3$ samples which have $(x_1,x_2)$. We take $n=n_1+n_2$ and index the samples as $(x^i_1,x^i_2)$ for $i=1,...,n$. After indexing the samples, we then define $H^{(\infty)}$ as a $n\times n$ matrix such that $H^{(\infty)}_{i,j}=h(x^i_1,x^j_1)$.

\begin{proof}[Proof of Proposition \ref{prop:nn}]
The detailed conditions for Proposition \ref{prop:nn} are as follows:
\begin{itemize}
    \item The learning rate $\eta=\Theta(n^{-\frac{3d_1-1}{2d_1-1}})$.
    \item The penalty $\lambda=\Theta( n^{\frac{d_1-1}{2d_1-1}})$.
    \item The number of hidden nodes $m\geq \tau^{-2}\text{poly}(n,1/\lambda_0)$ for some initialization variance $\tau^2=O(1)$ and $\lambda_0=\lambda_{\min}^{-1}(^{(\infty)})$.
    \item The number of iterations $T$ satisfies $\log(\text{poly}(n,\tau,1/\lambda_0))\ll \eta\mu T \ll \log(\text{poly}(\tau,1/n,\sqrt{m}))$.
    \item The input $x_1$ is normalized such that $\|x_1\|=1$, and this normalization does not change the minimal misclassification rate. 
    Denoting $\mu_1$ and $\mu_1'$ as the conditional expectation of $X_1$ (after normalization) under $y=\pm 1$, then both $\|\mu_1\|$ and $\|\mu_1'\|$ are nonzero. 
\end{itemize}

Since $d_2$ is a constant, training a neural network with input dimension $d_1$ and output dimension $d_2$ is equivalent to training $d_2$ different neural networks. Therefore following \cite{hu2021regularization}, one obtain that
\begin{eqnarray}\label{eqn:nn}
\|\widehat{\phi}-\phi^*\|_2^2=O_p\left( n^{-\frac{d}{2d-1}} \right).
\end{eqnarray}

Denote $(X_1,X_2)$ as the data matrix for $S_1, S_3$ (without response), and $(X_1',X_2',Y')$ as the data matrix for $S_1, S_2$.  

In terms of $\widehat{W}$, under CI, 
\begin{eqnarray}
\frac{\phi^*(X_1')^{\top}\phi^*(X_1')}{n_1+n_2}
\end{eqnarray}
is not a full rank matrix (at most rank two for binary classification). To avoid singular matrix problem, we take
\begin{eqnarray*}
&&\widehat{W}-W^*\\
&=&\lim_{\lambda\rightarrow 0}( \widehat{\phi}(X_1')^{\top} \widehat{\phi}(X_1')+\lambda I_{d_1})^{-1}\widehat{\phi}(X_1')^{\top}Y'-W^*\\
&=&\lim_{\lambda\rightarrow 0}\left( \frac{\phi^*(X_1')^{\top}\phi^*(X_1')}{n_1+n_2} +\frac{(\widehat{\phi}-\phi^*)^{\top}\phi^* }{n_1+n_2} +\frac{(\phi^*)^{\top}(\widehat{\phi}-\phi^*)}{n_1+n_2}+o +\lambda I_{d_1} \right)^{-1}\\&&\cdot\left( \frac{\phi^*(X_1')^{\top}Y'}{n_1+n_2}+ \frac{(\widehat{\phi}-\phi^*)^{\top}Y'}{n_1+n_2}+o\right)-W^*\\
&=&\lim_{\lambda\rightarrow 0} \left(\frac{\phi^*(X_1')^{\top}\phi^*(X_1')}{n_1+n_2}+\lambda I_{d_1}\right)^{-1}\left( \frac{(\widehat{\phi}-\phi^*)^{\top}\phi^* }{n_1+n_2} +\frac{(\phi^*)^{\top}(\widehat{\phi}-\phi^*)}{n_1+n_2} \right)\left(\frac{\phi^*(X_1')^{\top}\phi^*(X_1')}{n_1+n_2}+\lambda I_{d_1}\right)^{-1}\\&&\cdot\left( \frac{\phi^*(X_1')^{\top}Y'}{n_1+n_2} \right)\\
&&+\lim_{\lambda\rightarrow 0}\left(\frac{\phi^*(X_1')^{\top}\phi^*(X_1')}{n_1+n_2}+\lambda I_{d_1}\right)^{-1} \left( \frac{(\widehat{\phi}-\phi^*)^{\top}Y'}{n_1+n_2}\right)+o.
\end{eqnarray*}
As a result, $\widehat{W}\rightarrow W^*$.

Different from \cite{lee2020predicting}, we are considering the regret (the difference on the misclassification rate between the estimated classifier and the Bayes classifier) as the final performance measure. Based on the definition of $W^*$ and $\phi^*$, if we use $\sgn( \widetilde{W}^{\top} \phi^*(x_1))$ for some estimate $\widetilde{W}$ such that $\widetilde{W}\rightarrow W$, the classifier always makes the exact same decision as the Bayes classifier.

On the other hand, for the estimated output
\begin{eqnarray*}
\widehat{W}^{\top}\widehat{\phi}(x_1)=\widehat{W}^{\top}\phi^*(x_1)+(W^*)^{\top}( \widehat{\phi}(x_1)-\phi^*(x_1) )+o,
\end{eqnarray*}
since we have argued that $\sgn(\widehat{W}^{\top}\phi^*(x_1))\equiv \sgn((W^*)^{\top}\phi^*(x_1))$, we aims to study how $(W^*)^{\top}( \widehat{\phi}(x_1)-\phi^*(x_1) )$ affects the regret.

The regret can be represented as
\begin{eqnarray*}
&&\int |1/2-p(x_1)| 1\left\{\sgn(\widehat{W}^{\top}\widehat{\phi}(x_1))\neq \sgn((W^*)^{\top}\phi^*(x_1))\right\}  dP(x_1)\\
&=& \int |1/2-p(x_1)| 1\left\{\sgn\left(\widehat{W}^{\top}\phi^*(x_1)+(W^*)^{\top}( \widehat{\phi}(x_1)-\phi^*(x_1) )\right)\neq \sgn\left((W^*)^{\top}\phi^*(x_1)\right)\right\}  dP(x_1) +o \\
&\leq& \int |1/2-p(x_1)| 1\left\{ \|W^*\| \| \widehat{\phi}(x_1)-\phi^*(x_1) )\| \geq \min\left(|\widehat{W}^{\top}\phi^*(x_1)|,|(W^*)^{\top}\phi^*(x_1)|\right) \right\}  dP(x_1) +o \\
&\leq&(1+o(1))\int |1/2-p(x_1)| 1\left\{\| \widehat{\phi}(x_1)-\phi^*(x_1) )\| \geq |(W^*)^{\top}\phi^*(x_1)|/\|W^*\|\right\}  dP(x_1) +o.
\end{eqnarray*}
Further, since $(W^*)^{\top}\phi^*(x_1)=p(x_1)(W^*)^{\top}\mu_1+(1-p(x_1))(W^*)^{\top}\mu_1'$ and $\|\mu_1\|,\|\mu_1'\|=\Theta(1)$, there exists some $c>0$ such that
\begin{eqnarray*}
&&\int |1/2-p(x_1)| 1\left\{\| \widehat{\phi}(x_1)-\phi^*(x_1) )\| \geq |(W^*)^{\top}\phi^*(x_1)|/\|W^*\|\right\}  dP(x_1) \\
&\leq& \int |1/2-p(x_1)| 1\left\{\| \widehat{\phi}(x_1)-\phi^*(x_1) )\| \geq c|1/2-p(x_1)|\right\}  dP(x_1)\\
&\leq& \frac{1}{c}\int \| \widehat{\phi}(x_1)-\phi^*(x_1) )\|  1\left\{\| \widehat{\phi}(x_1)-\phi^*(x_1) )\| \geq c|1/2-p(x_1)|\right\}  dP(x_1)\\
&\leq& \frac{1}{c}\sqrt{ \int \| \widehat{\phi}(x_1)-\phi^*(x_1) )\|^2  dP(x_1)},
\end{eqnarray*}
which becomes $O_p( (n_1+n_3)^{ -d/2(2d-1)})$ based on (\ref{eqn:nn}).
\end{proof}

\subsection{Proof for Section \ref{sec:appendix:regression}}

\begin{proof}[Proof of Theorem \ref{thm:regression}, Linear $\phi$, Regression, Lower Bound]
Assume $\Sigma_{1,1}=Var(x_1)=I_{d_1}$ and $\Sigma_{2,2}=Var(x_2)=I_{d_2}$. When using linear $\phi$, it is easy to see that $\phi^*(x_1) = \Sigma_{2,1}x_1 $ for $\Sigma_{2,1}=Cov(x_2,x_1)$. Thus when SSL is unbiased, we have $\theta_0x_1\equiv a^{\top}\Sigma_{2,1}x_1$, i.e. $\theta_0=a^{\top}\Sigma_{2,1}$ for some vector $a$. From Lemma \ref{lem:minimax:covariance}, we have known fixed $a$,
\begin{eqnarray*}
\inf_{\widehat{\theta}}\sup_{\Sigma_{1,2}}\mathbb{E}\|\widehat{\theta}-\Sigma_{1,2}a\|^2=\Omega\left(\frac{d_1+d_2}{n_1+n_3}\|a\|^2\right),
\end{eqnarray*}
and for known $\Sigma_{1,2}$, from Lemma \ref{lem:minimax:a} we have
\begin{eqnarray*}
\inf_{\widehat{\theta}}\sup_{a}\mathbb{E}\|\widehat{a}-a\|^2=\Omega\left(\frac{\sigma^2 d_2}{n_1+n_2}\right).
\end{eqnarray*}
Using the strong convexity property of the risk, we then obtain
\begin{eqnarray*}
\inf_{\widehat{\theta}}\sup_{\Sigma_{1,2}}\mathbb{E}R(\widehat{\theta})-R^*=\Omega\left(\frac{d_1+d_2}{n_1+n_3}\|a\|^2+\frac{\sigma^2 d_2}{n_1+n_2}\right),
\end{eqnarray*}
where $\|a\|=\Theta(\|\theta_0\|)$ based on assumption.
\end{proof}
\begin{proof}[Proof of Theorem \ref{thm:regression}, Linear $\phi$, Regression, Upper Bound]
The proof is similar to the square loss case in Theorem \ref{thm:linear}. Denote $(X_1,X_2)$ as the data matrix for $S_1, S_3$ (without response), and $(X_1',X_2',Y')$ as the data matrix for $S_1, S_2$. Denote $\Sigma_{i,j}=\mathbb{E}x_ix_j^{\top}$ for $i,j\in\{1,2\}$. Also denote $\Sigma_{i,y}$ as $\mathbb{E}x_iy$ for $i\in\{1,2\}$. 

We first look at the asymptotics of $\widehat{\phi}\widehat{W}$. From the problem setup, we can directly solve $\widehat{\phi}$:
\begin{eqnarray*}
\widehat{\phi} = (X_1^{\top}X_1)^{-1}X_1^{\top}X_2 \rightarrow \Sigma_{1,1}^{-1}\Sigma_{1,2},
\end{eqnarray*}
and further write down $\widehat{W}$:
\begin{eqnarray*}
\widehat{W}=( \widehat{\phi}(X_1')^{\top} \widehat{\phi}(X_1'))^{-1}\widehat{\phi}(X_1')^{\top}Y'\rightarrow (\Sigma_{2,1}\Sigma_{1,1}^{-1}\Sigma_{1,2})^{-1}\Sigma_{2,1}\theta_0.
\end{eqnarray*}
Thus
\begin{eqnarray*}
\widehat{\theta}=\widehat{\phi}\widehat{W}\rightarrow \Sigma_{1,1}^{-1}\Sigma_{1,2}(\Sigma_{2,1}\Sigma_{1,1}^{-1}\Sigma_{1,2})^{-1}\Sigma_{2,1}\theta_0.
\end{eqnarray*}
When $\theta_0=\Sigma_{1,1}^{-1}\Sigma_{2,1}a_0$ for some $a_0$, we have
\begin{eqnarray*}
\Sigma_{1,1}^{-1}\Sigma_{1,2}(\Sigma_{2,1}\Sigma_{1,1}^{-1}\Sigma_{1,2})^{-1}\Sigma_{2,1}\theta_0=\theta_0,
\end{eqnarray*}
i.e., SSL is unbiased.

We next study the convergence rate. Denote $\theta_0$ as the $\widehat\theta$ obtained when $n_i\rightarrow\infty$ for all $i=1,2,3$. For classification task, there is no preference on the magnitude of $\theta_0$ as it works as a linear classifier, so we take $\|\theta_0\|=1$. For the pretext task, one can see that
\begin{eqnarray*}
\widehat{\phi}-{\phi}^*&=&( \widehat{\phi}(X_1')^{\top} \widehat{\phi}(X_1'))^{-1}\widehat{\phi}(X_1')^{\top}Y'-{\phi}^*\\
&=&\Sigma_{1,1}^{-1}\left(\frac{X_1^{\top}X_2}{n_1+n_3}-\Sigma_{1,2}\right)+\Sigma_{1,1}^{-1}\left(\frac{X_1^{\top}X_1}{n_1+n_3}-\Sigma_{1,1}\right)\Sigma_{1,1}^{-1}\Sigma_{1,2}+o.
\end{eqnarray*}
\begin{eqnarray*}
\widehat{W}-W^*&=&( \widehat{\phi}(X_1')^{\top} \widehat{\phi}(X_1'))^{-1}\widehat{\phi}(X_1')^{\top}Y'-W^*\\
&=&\left( \frac{\phi^*(X_1')^{\top}\phi^*(X_1')}{n_1+n_2} + (\widehat{\phi}-\phi^*)^{\top}\Sigma_{1,2}+\Sigma_{2,1}(\widehat{\phi}-\phi^*)+o  \right)^{-1}\\&&\cdot\left( \frac{\phi^*(X_1')^{\top}Y'}{n_1+n_2}+ (\widehat{\phi}-\phi^*)^{\top}\Sigma_{1,1}\theta_0 +o\right)-W^*,
\end{eqnarray*}
where
\begin{eqnarray*}
&&\left( \frac{\phi^*(X_1')^{\top}\phi^*(X_1')}{n_1+n_2} + (\widehat{\phi}-\phi^*)^{\top}\Sigma_{1,2}+\Sigma_{2,1}(\widehat{\phi}-\phi^*)+o  \right)^{-1}-\left(\frac{\phi^*(X_1')^{\top}\phi^*(X_1')}{n_1+n_2}\right)^{-1}\\
&=&\left(\frac{\phi^*(X_1')^{\top}\phi^*(X_1')}{n_1+n_2}\right)^{-1}\left( (\widehat{\phi}-\phi^*)^{\top}\Sigma_{1,2}+\Sigma_{2,1}(\widehat{\phi}-\phi^*) \right)\left(\frac{\phi^*(X_1')^{\top}\phi^*(X_1')}{n_1+n_2}\right)^{-1}+o.
\end{eqnarray*}
As a result, denoting $\widehat{\theta}=\widehat{\phi}^{\top}\widehat{W}$, we have
\begin{eqnarray*}
\widehat{\theta}-\theta_0 &=& (\widehat{\phi}-\phi^*)^{\top}W^*+(\phi^*)^{\top}(\widehat{W}-W^*)+o,
\end{eqnarray*}
therefore, 
\begin{eqnarray*}
\mathbb{E}\|\widehat{\theta}-\theta_0\|^2=O\left( \frac{d_2\sigma_2^2}{n_1+n_2} + \frac{d_1+d_2}{n_1+n_3}\|\theta_0\|^2 \right).
\end{eqnarray*}
\end{proof}
\begin{proof}[Proof of Theorem \ref{thm:regression}, Linear $\phi$, Regression, Adversarial]
For the convergence upper bound, following the decomposition of estimation error in \cite{xing2021adversarially}, beside the part from $\|\widehat{\theta}-\theta_0\|^2$ in clean training, there is one extra part due to the information limit on $\Sigma_{1,1}$. However, since there are $\sum n_i$ samples to provide information of $x_1$, the new term can be ignored, so the minimax lower bound in adversarial training is the same to clean training.

In terms of the lower bound, we know that 
\begin{eqnarray*}
\theta(\epsilon)=(\Sigma_{1,1}+\lambda I_{d_1})^{-1}\Sigma_{1,1}\theta(0),
\end{eqnarray*}
thus following \cite{xing2021adversarially}, we consider two scenarios: (3) $\Sigma_{1,1}$ is known and we impose prior distribution on $\theta(0)$; (4) $\theta(0)$ is known and we impose prior distribution on $\Sigma_{1,1}$. Following the arguments in clean training, we have scenario (3) reduces to clean training setup. For scenario (4), following \cite{xing2021adversarially} we obtain
\begin{eqnarray*}
\inf_{\widehat\theta}\sup_{\Sigma,a,\epsilon}\mathbb{E}\|\widehat\theta-\theta(\epsilon)\|^2=\Theta\left(\frac{d_1}{n_1+n_2+n_3+n_4}\right).
\end{eqnarray*}

To conclude,
\begin{eqnarray*}
\inf_{\widehat\theta}\sup_{\Sigma,a,\epsilon}\mathbb{E}\|\widehat\theta-\theta(\epsilon)\|^2=\Theta\left( \frac{\sigma^2 d_2}{n_1+n_2}+\frac{d_1}{n_1+n_3}+\frac{d_1}{n_1+n_2+n_3+n_4} \right)=\Theta\left(  \frac{\sigma^2 d_2}{n_1+n_2}+\frac{d_1}{n_1+n_3}\right).
\end{eqnarray*}
\end{proof}



    

\end{document}